\documentclass{article}
\usepackage{microtype}
\usepackage{graphicx}
\usepackage{subfig}
\usepackage{booktabs} 
\usepackage{microtype}
\usepackage{graphicx}
\usepackage{xcolor}
\usepackage{booktabs} 
\usepackage[utf8]{inputenc} 
\usepackage[T1]{fontenc}    
\usepackage{url}            
\usepackage{caption}
\usepackage{amsfonts}       
\usepackage{nicefrac}       
\usepackage{microtype}      
\usepackage{bm}
\usepackage{diagbox}
\usepackage{color} 
\usepackage{amssymb}
\usepackage{amsmath}
\usepackage{amsfonts}
\usepackage{mathtools}
\usepackage{lineno}
\usepackage[toc,page,header]{appendix}
\usepackage{minitoc}

\usepackage{mathrsfs}
\usepackage{amsthm}
\usepackage{amsmath}
\usepackage{amssymb, bm}
\usepackage{comment}
\usepackage{multirow}
\usepackage{tabularx}
\usepackage{ulem}
\usepackage{siunitx}
\usepackage[version=4]{mhchem}

\usepackage[preprint]{neurips_2024}
\usepackage{hyperref}

\makeatother

\usepackage{wrapfig}

\usepackage{multicol, multirow}
\usepackage[ruled,vlined,algo2e]{algorithm2e}

\usepackage{algorithmic}
\usepackage{algorithm}

\newtheorem{theorem}{Theorem}
\newtheorem{definition}{Definition}
\newtheorem{proposition}{Proposition}

\newtheorem{lemma}{Lemma}

\newcommand{\cA}{{\mathcal A}}

\newcommand{\cG}{{\mathcal G}}

\newcommand{\cP}{{\mathcal P}}

\newcommand{\cS}{{\mathcal S}}

\newcommand{\cY}{{\mathcal Y}}
\newcommand{\cX}{{\mathcal X}}

\SetKwInput{KwInput}{Input}                
\SetKwInput{KwOutput}{Output}

\providecommand{\customgenericname}{}
\newcommand{\newcustomtheorem}[2]{%
  \newenvironment{#1}[1]
  {%
  \renewcommand\customgenericname{#2}%
  \renewcommand\theinnercustomgeneric{##1}%
  \innercustomgeneric
  }
  {\endinnercustomgeneric}
}
\newcustomtheorem{customtheorem}{Theorem}

\usepackage{hyperref}

\usepackage{hyperref}

\usepackage{amsmath}
\usepackage{amssymb}
\usepackage{mathtools}
\usepackage{amsthm}

\usepackage[capitalize,noabbrev]{cleveref}

\theoremstyle{plain}
\theoremstyle{definition}

\theoremstyle{remark}

\usepackage[textsize=tiny]{todonotes}
\title{All Language Models Large and Small}

\author{Zhixun Chen\\
Kings College London\\
\And
Yali Du\\
Kings College London\\
\texttt{yali.du@kcl.ac.uk}\\
\And
David Mguni\\
Queen Mary University, London\\
\texttt{davidmguni@hotmail.com}
}

\begin{document}

\maketitle
\begin{abstract}\label{sec:abs}
Many leading language models (LMs) use high-intensity computational
resources both during training and execution. This poses the challenge of lowering resource costs for deployment and faster execution in decision-making tasks among others. We introduce a novel plug \& play LM framework named \textbf{L}anguage \textbf{O}ptimising \textbf{N}etwork \textbf{Di}stribution (LONDI). LONDI learns to selectively employ large LMs only where complex decision-making and reasoning are required while using low-resource LMs (i.e. LMs require less GPU usage, but may not be able to solve the problem alone) everywhere else. LONDI consists of a system of two (off-)policy networks, an LM, a large LM (LLM), and a reinforcement learning module that uses \textit{switching controls} to quickly learn in which system states to call the LLM.  We then introduce a variant of LONDI that maintains budget constraints on LLM calls and hence its resource usage. Theoretically, we prove LONDI learns the subset of system states to activate the LLM required to solve the task. We then prove that LONDI converges to optimal solutions while also preserving budgetary constraints on LLM calls almost surely enabling it to solve various tasks while significantly lowering computational costs. We test LONDI's performance in a range of tasks in ScienceWorld and BabyAI-Text and demonstrate that LONDI can solve tasks only solvable by resource-intensive LLMs while reducing GPU usage by up to 30\%.
\end{abstract}




\section{Introduction}\label{sec:intro}
Large language models (LLMs) have emerged as powerful tools that can assist humans to accomplish a wide range of tasks such as medical tasks \cite{lin2023pushing}, language education \cite{caines2023application}, autonomous driving \cite{fu2023drive} and recreational games \cite{feng2024chessgpt}. As LLMs originate from data center warehouses, they are expected to gradually extend their reach to edge devices such as personal computers, smartphones, and even Internet of Things (IoT) devices \cite{huawei24,qualcomm24}. This shift is driven by the desire for enhanced data privacy, availability of AI functionalities, and personalised experiences \cite{yi2023edgemoe}. For instance, Qualcomm has effectively implemented stable diffusion, a text-to-image generative LLM model, on smartphones \cite{qualcomm24}. Multimodal LLMs have been integrated into smartphones, enabling precise content searching through natural language queries \cite{huawei24}.

However, current LLMs require massive computational resources for training and execution, which makes the application of LLMs in edge devices without large memory storage still unrealistic \cite{ahmed2020dedemocratization,yi2023edgemoe}. Smaller language models (LMs) may offer a practical solution to the computational constraints. Whereas the ability of LLM is significantly correlated with the size of its parameter set \cite{brown2020language}. The performance of smaller LMs is dramatically degraded compared with larger LMs, rendering smaller LMs unable to solve some tasks \cite{ding2023efficiency}. A practical solution is to combine a large LM with a small LM and activate the large LM only under a set of conditions. Though there have been several attempts to combine LMs in this way, the challenge of systematically constructing a set of conditions to activate LMs has not yet been fully resolved.

Dual process theory (DPT) \cite{Kahneman11} posits that human thinking and decision-making involve two separate cognitive processes or systems, with one being characterised as fast, automatic, and intuitive, while the other is described as slow, controlled, and reflective. According to this theory, this delineation enables fast, reactive decision-making in states where complex modes of thought are not required while employing slower yet more sophisticated, complex reasoning where it is required.  

Inspired, by DPT, we propose a dual language model structure that consists of two LMs and an adaptive reinforcement learning (RL) agent, {\fontfamily{cmss}\selectfont Switcher} as a switch mechanism to automate selective activations of the LMs. By applying a smaller LM (which we call QUICK) as a fast response module and a larger LLM (which we call DEEPTHINK) as the slow thinking module, LONDI's adaptive switch mechanism learns to efficiently manage cooperative delegation between two LMs to solve tasks while reducing the resource cost burden of LMs. Specifically, the switch agent, based on switching control policy \cite{mguni2023learning,mguni2022timing,pmlr-v202-mguni23a}, determines the states in which to activate DEEPTHINK, while QUICK is utilised in all other states. Therefore, we approach the problem of computational constraints within LLMs from a systematic standpoint using the switching control, which proves the optimality of our framework and convergence to a policy that only activates DEEPTHINK at the beneficial set of states. LONDI has a cost parameter $c$ , which plays an important role in calibrating the resource-consumption/performance trade-off of the system of MARL. Larger values of $c$ incur higher costs for each activation of the DEEPTHINK large language model by the {\fontfamily{cmss}\selectfont Switcher}. This in turn makes the {\fontfamily{cmss}\selectfont Switcher} more selective about activating the DEEPTHINK LLM, reserving its activations to a smaller number of states where the boost in performance is greatest (this is proven in  Theorem \ref{theorem:existence}). In the limit $c\to \infty$ the Switcher becomes extremely thrifty in which case LONDI solely uses the QUICK model. This enables the resource-consumption/performance trade-off to be easily calibrated by the user. Moreover, in Sec. \ref{sec:londi_budget} we introduce a variant of LONDI, namely LONDI-B that imposes a budgetary constraint on the number of DEEPTHINK calls and prove that LONDI optimises its use of DEEPTHINK while respecting the budget constraint almost surely.

\begin{figure}[!htp]
  \centering
  \subfloat[Task Map]{\includegraphics[width=0.20\textwidth,height=0.19\textwidth]{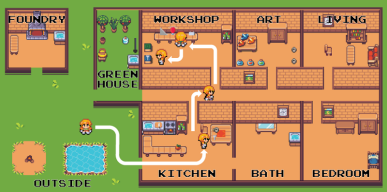}\label{fig_sw}}
  \subfloat[budget=2]
  {\includegraphics[width=0.20\textwidth,height=0.19\textwidth]{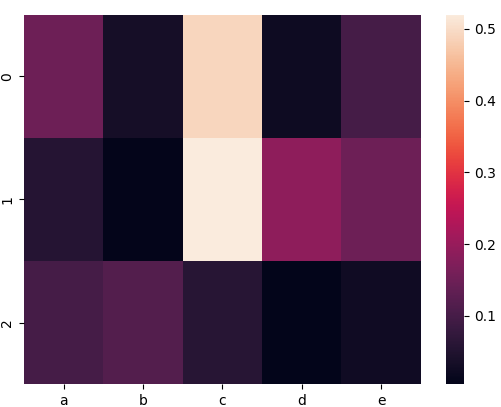}\label{fig:subfig7}}
  \subfloat[budget=4]
  {\includegraphics[width=0.20\textwidth,height=0.19\textwidth]{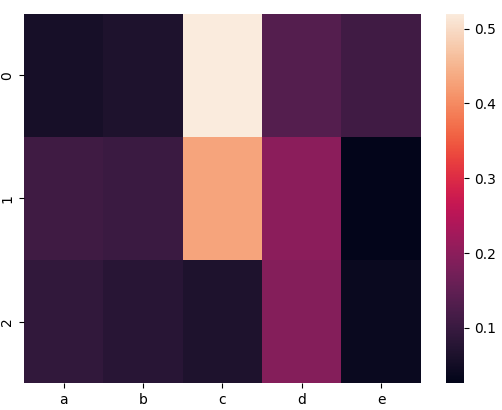}\label{fig:subfig8}}
  \subfloat[budget=6]
  {\includegraphics[width=0.20\textwidth,height=0.19\textwidth]{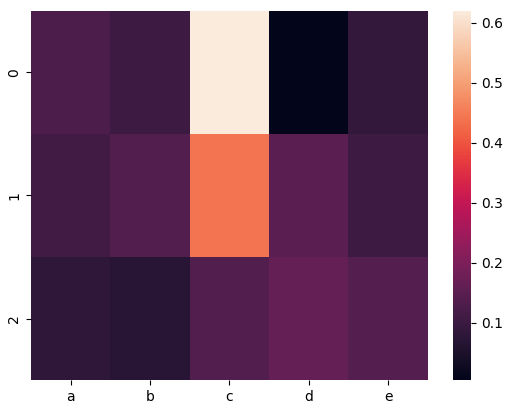}\label{fig:subfig9}}
  \subfloat[budget=8]
  {\includegraphics[width=0.20\textwidth,height=0.19\textwidth]{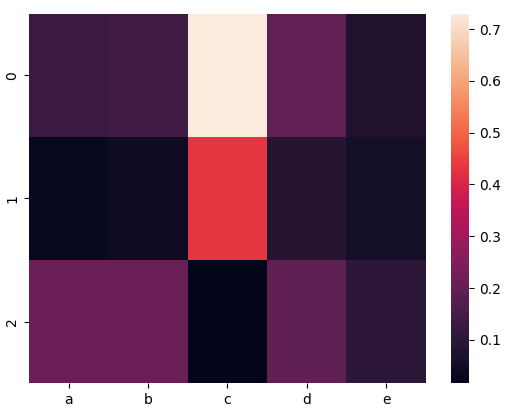}\label{fig:subfig10}}
  \caption{(a). The ScienceWorld task, Create a circuit. To complete the task, the agent must navigate to the hallway first and then determine the correct room to enter then use the material present in the room. (b),(c),(d),(e) sub-figures are the heatmap of LONDI DEEPTHINK calls with different budget on ScienceWorld task, Create a circuit. Since the agent needs to pass through the hallway to reach the workshop initially, LONDI must activate DEEPTHINK model in the hallway at least once. Subsequently, the remaining budget allocated to activating the DEEPTHINK model at the workshop to accomplish the task and obtain a higher reward. Therefore, as the budget decrease, LONDI focuses its activations of DEEPTHINK solely at the critical parts of the map such as hallway, resulting in lighter colors in the heatmap.}
  \label{hp_dif_budget}
\end{figure}

In practical settings, adjusting the cost $c$ to align with the available computational resources and communication constraints \cite{ahmed2020dedemocratization} maybe time-consuming. To accommodate such instances, we extend LONDI to LONDI-B in Sec.\ref{sec:londi_budget}, which introduces a budget facility into the LONDI framework. In this setup, LONDI maintains a budget on the number of DEEPTHINK calls allowed. Our theory then proves that LONDI-B preserves the budget constraint almost surely (with probability 1) while ensuring that DEEPTHINK is activated at the most beneficial set of states given the budget on the number of DEEPTHINK calls (see Sec. \ref{sec:convergence}).
Overall, LONDI has several advantages:  
%
\newline
    $\bullet$ By switching to DEEPTHINK only at states where it is beneficial while leveraging the computational thriftiness of QUICK, LONDI solves various tasks while reducing computational expense (see Sec. \ref{results_main}, Appendix. \ref{sec:abl_results}). 
    \newline
    $\bullet$ LONDI can preserve fixed budgets on the number of DEEPTHINK calls, balancing performance and computational cost under the limited budgets.(see Sec. \ref{results_main}).
    \newline
    $\bullet$ LONDI is a plug-and-play framework that seamlessly adopts any QUICK and DEEPTHINK module. (see Appendix. \ref{sec:abl_results}). 
\section{Related Work}\label{sec:related work}
Recently, various tools have been proposed to augment the capabilities of LLMs. SayCan \cite{ahn2022i} employs an LLM with an additional value function to assign scores to high-level actions, and then utilises a low-level planner to map these actions to determine their feasibility in the physical world. \cite{Lin_Huang_Liu_Gu_Sommerer_Ren_2023} proposes an encoder-decoder structure to facilitate the planning ability of LM. \cite{pmlr-v162-huang22a} decomposes tasks into mid-level plans and maps outputs to available actions. DEST \cite{wang2023describe} applies a self-explanation mechanism for error correction and a goal-selector to rank sub-goals. Combining with PDDL, \cite{guan2023leveraging} utilises LLM to generate, translate and validate PDDL models to address planning tasks. Combined with reinforcement learning, GFlan \cite{carta2023grounding} employs Flan-T5 \cite{chung2022scaling} as the action policy and updates it with online PPO algorithm \cite{schulman2017proximal}. However, all these methods encounter the problem of high computational resource cost.  To address the challenge of high computational costs, in this paper, we introduce a switching mechanism within a dual LM structure delineated by a low-cost and high-cost LM. 

Closest to our work is the SWIFTSAGE \cite{lin2023swiftsage} framework, which combines a small LM module as the fast system and a large LLM module as the slow system. By combining two LLMs with varying sizes and computing power, the framework tackles intricate interactive reasoning tasks while mitigating the computational load. Although it achieves remarkable performance with GPT-4, the method to interpolate between the two modules uses a hand-crafted heuristic protocol which can lead to suboptimal performance (see Sec. \ref{results_main}). Another similar work is the FrugalGPT \cite{chen2023frugalgpt}, which combines a cascade of LLMs and a score function to decide which LLM to use. However, defining an appropriate score function to guide the LLM selection process is challenging in complex planning tasks. Therefore, the computational constraint problem is imperfectly resolved in both cases. In comparison, LONDI uses reinforcement learning in conjunction with a type of policy known as \textit{switching controls} to learn at which system states the DEEPTHINK module should be activated. Moreover, this systematic learning approach to the LLM activation enables a variant of LONDI to maintain a budget constraint on the number of DEEPTHINK calls (see Sec. \ref{sec:londi_budget}).

The switching structure is similar to the mechanism of a psychological framework, dual process theory \cite{Wason1975DualPI, Kahneman11}.  
%
Dual process structures have inspired various mechanisms in reinforcement learning used to improve learning efficiency. ROSA \cite{mguni2023learning} and LIGS \cite{mguni2022ligs} incorporate a dual switching method to activate a reward-shaping module to promote state visitations and coordination between adaptive agents in an RL and MARL respectively. LICRA \cite{mguni2022timing} adds a trainable switch to decide whether to use a costly execution-policy system to generate actions. Similarly, MANSA \cite{pmlr-v202-mguni23a} has an additional switch to decide whether to activate centralised training, a computationally expensive learning mode that facilitates coordination among adaptive agents. 

\section{LONDI}
We now describe the problem setting, details of our framework, and how it learns to select where to activate the DEEPTHINK  large language model. We then describe {\fontfamily{cmss}\selectfont Switcher}'s objective and the switching control mechanism it uses to learn where to activate the DEEPTHINK model.   
\subsection{Markov Decision Process}
In this paper, we consider a setting in which an agent is tasked with solving a decision-making problem.  We give some preliminaries on Markov decision processes which is the underlying formalism for our problem. A Markov decision process (MDP) \cite{puterman2014markov} is given by a tuple $\left\langle\mathcal{S,A,P,R,}\gamma\right.\rangle$ where $\mathcal{S}$ represents the set of states, $\mathcal{A}$ means the set of (discrete) actions, the transition probability $\mathcal{P}:\mathcal{S}\times\mathcal{A}\times\mathcal{S}\rightarrow\left[0,1]\right.$ indicates the system dynamics, the reward function $\mathcal{R}:\mathcal{S}\times\mathcal{A}\rightarrow\mathbb{R}$ describes the performance of the agent, the reward discount factor $\gamma\in[0,1]$ defines the level of discount applied to the rewards. At time $t$, the system is at state $s_t\in\mathcal{S}$ and the agent decides on its action $a_t\in\mathcal{A}$ using the policy $\pi: \mathcal{S}\times\mathcal{A}\rightarrow[0,1]$, where $\pi(a|s)$ is the probability of choosing action $a\in\mathcal{A}$ under state $s\in\mathcal{S}$. The action transitions the system to a new state $s_{t+1}\sim\mathcal{P}(\cdot|s_t,a_t)$ and the agent then receives a scalar reward $r\sim\mathcal{R}(s_t,a_t)$. In the standard setup of an MDP,
the agent's objective is to maximise the cumulative expected rewards $v^{\pi}(s):=\mathbb{E}[\sum^{+\infty}_{t=0}\gamma^{t}R(s_t,a_t)|a_t\sim\pi(\cdot|s_t)]$ using a policy $\pi^{\ast}\in\Pi$ where $\Pi$ is the policy set of the agent. 
\subsection{Problem Formulation}
To tackle the challenges described in Sec. \ref{sec:related work}, we introduce an adaptive learner which we call {\fontfamily{cmss}\selectfont Switcher} that decides on the states to activate DEEPTHINK while using the less computationally expensive QUICK language model to determine actions everywhere else. The {\fontfamily{cmss}\selectfont Switcher} needs to make a binary decision (whether to activate DEEPTHINK or not) at each system state, where a state in the current setting is a representation of a specific \textit{context} or \textit{condition}. The choice to activate a larger LLM can be seen as an action that transitions the system to a new state, which returns a reward or penalty. Lastly, the {\fontfamily{cmss}\selectfont Switcher} has an objective quantified by the expected return. 

We can formalise the {\fontfamily{cmss}\selectfont Switcher} problem as an MDP as the problem involves sequential decision-making (under uncertainty) with Markovian transitions. Specifically, $\mathcal{S}$ represents the system state space, $\boldsymbol{\mathcal{A}}\equiv\mathcal{A}_S\cup\mathcal{A}$ denotes the action set where action is executed in the environment and which is decided by an `active' language model, $\mathcal{A}_S\equiv \{0,1\}$ denotes the {\fontfamily{cmss}\selectfont Switcher}'s binary action set, $\mathcal{A}$ represents the action set common to both language models, $\mathcal{P}$ indicates the transition after the action and $\mathcal{R}:\mathcal{S}\times\boldsymbol{\mathcal{A}}\to \mathbb{R}$ represents the return of the environment after the action of LLM activated by the {\fontfamily{cmss}\selectfont Switcher}. At any given instant, only one of the DEEPTHINK and QUICK modules is activated and hence able to take action. To make a decision, the {\fontfamily{cmss}\selectfont Switcher} samples a decision $g$ from its policy $\mathfrak{g}:\mathcal{S}\rightarrow\{0,1\}$ where $g=1$ indicates an activation of the DEEPTHINK module in which case the action $a\sim\pi^{\rm DEEP}$ is executed while $g=0$ indicates that no activation of the DEEPTHINK module occurs so the QUICK module is active in which case the action $a\sim\pi^{\rm switch}$ is executed where $\pi^{\rm QUICK}$ and $\pi^{\rm DEEP}$ are policies associated to the QUICK and DEEPTHINK modules respectively.

Under the MDP setting, the goal of the {\fontfamily{cmss}\selectfont Switcher} is to maximise the cumulative return, namely the overall performance of the structure. To prompt {\fontfamily{cmss}\selectfont Switcher} to make selective activations decisions, a fixed cost associated with each activation is imposed on {\fontfamily{cmss}\selectfont Switcher}, represented by a constant value $c<0$. The incurred costs encourages the {\fontfamily{cmss}\selectfont Switcher} to activate DEEPTHINK model only when the activation is advantageous for the system's performance, either in the current state or in subsequent states. The objective of the {\fontfamily{cmss}\selectfont Switcher} policy $\mathfrak{g}$ is 
\begin{align*}
\hspace{-3 mm}v_S(s|{\pi},\mathfrak{g})  = \mathbb{E}_{g\sim\mathfrak{g}}\left[ \sum_{t=0}^\infty \gamma^t\left(r -c\cdot {1}(g(s_t))\right)\Big|s_0=s; a_t\sim\pi\right]
\end{align*}
and the action-value function of it is $Q_{S}(s,{a}|{\pi},{\mathfrak{g}})=\mathbb{E}_{g\sim\mathfrak{g}}[\sum^{\infty}_{t=0}\gamma^t(r-c\cdot\textbf{1}(g(s_t)))|s_0=s;{a_0}={a};{a_t}\sim{\mathfrak{\pi}}]$, where $\pi$ is either $\pi^{\rm QUICK}$ or $\pi^{\rm DEEP}$.
With this objective, {\fontfamily{cmss}\selectfont Switcher}'s goal is to maximise the system performance by activating DEEPTHINK at the required set of states to enable the task to be solved with the minimal number of DEEPTHINK activations. Therefore, by learning an optimal $\mathfrak{g}$, {\fontfamily{cmss}\selectfont Switcher} acquires the optimal policy for activating  DEEPTHINK.



Adding the agent {\fontfamily{cmss}\selectfont Switcher} an additional learning process with a distinct objective can lead to non-convergence among some methodologies \cite{shoham2008multiagent}.\footnote{In particular, it results in a non-cooperative Markov game.} We nevertheless prove in Sec. \ref{sec:convergence} the convergence of LONDI under standard RL assumptions.

\subsection{Switching Controls} 
The {\fontfamily{cmss}\selectfont Switcher} is tasked with learning the set of states that require the additional decision capacity provided by the DEEPTHINK model in order to achieve the optimal policy. To do this, at each state {\fontfamily{cmss}\selectfont Switcher} first makes a \textit{binary decision} to decide whether to activate its DEEPTHINK.  Switching controls enable {\fontfamily{cmss}\selectfont Switcher} to learn at which states it ought to activate the DEEPTHINK model. Therefore, in LONDI, the {\fontfamily{cmss}\selectfont Switcher} agent uses a form of policies known as \textit{switching controls} \cite{mguni2022timing,mguni2018viscosity}.
This leads to an RL problem in which, unlike the standard setup of an MDP, the {\fontfamily{cmss}\selectfont Switcher} agent now uses \textit{switching controls} to select its decisions. 

\textbf{{Summary of events}}\newline
At a time $t\in0,1...$\newline
$\bullet$ $\mathsf{Encoder}$ process the state $s_t\in\mathcal{S}$ \newline
$\bullet$ {\fontfamily{cmss}\selectfont Switcher} decides whether to activate the DEEPTHINK model according to the decision $g\sim\mathfrak{g}:\mathcal{S}\rightarrow\{0,1\}$:\newline
$\bullet$ if $g=0:$\newline
$\circ$ The QUICK is activated and samples an action $a_t$ from its policy $\pi^{\rm QUICK}$. The switch receives a reward $r\sim\mathcal{R}(s_t,a_t)$ and the system shift to the subsequent state $s_{t+1}$\newline
$\bullet$ if $g=1:$\newline
$\circ$ DEEPTHINK is activated and samples an action $a_t$ from its policy $\pi^{\rm DEEP}$. {\fontfamily{cmss}\selectfont Switcher} receives a reward $r-c$ where $r\sim\mathcal{R}(s_t,a_t)$ and the system transitions to the state $s_{t+1}$.

We now describe how at each state {\fontfamily{cmss}\selectfont Switcher} decides whether to activate DEEPTHINK. At any $s_t$, the decision to turn the DEEPTHINK model is decided by a (categorical) policy $\mathfrak{g}:\mathcal{S} \to \{0,1\}$. 
We denote by $\{\tau_k\}$ the times that activation takes place, for example, if the DEEPTHINK model is first activated at state $s_5$ then turned off at $s_7$, then $\tau_1=5$ and $\tau_2=7$. Recalling the role of $\mathfrak{g}$, the switching times obey the expression $\tau_k=\inf\{t>\tau_{k-1}|s_t\in\mathcal{S},\mathfrak{g}(s_t)=1\}$ and are therefore
\textit{\textbf{rules} that depend on the state.}. The termination times $\{\tau_{2k-1}\}$ occur according to some external (probabilistic) rule i.e., if at state $s_t$ DEEPTHINK is active, then DEEPTHINK deactivates at state $s_{t+1}$ with probability $p\in ]0,1]$.  Hence, by learning an optimal $\mathfrak{g}$, {\fontfamily{cmss}\selectfont Switcher} learns the best states to activate DEEPTHINK.

\subsection{Architecture}
We now describe a concrete realisation of LONDI's core components which consist of $2$ language models, a large language model (LLM) {\fontfamily{cmss}\selectfont DEEPTHINK}, a small language model as  {\fontfamily{cmss}\selectfont QUICK} and a switching control RL algorithm as {\fontfamily{cmss}\selectfont Switcher}. Each component (including the LLMs) can be replaced by various other components.
\\
    $\bullet$ \textbf{QUICK model.} In this paper, we use a pretrained Flan-T5-small model \cite{chung2022scaling} as the QUICK module.
    \newline
    $\bullet$ \textbf{DEEPTHINK model}. We use a pre-trained Flan-T5-large model \cite{chung2022scaling} as the DEEPTHINK module. It performs twice as well as the QUICK module on average.\footnote{The performance comparison is under our experiment setting, not the general comparison in all environments.} 
    \newline
    $\bullet$ \textbf{Switching Control Policy \cite{mguni2022timing}.} A soft actor-critic (SAC) \cite{haarnoja2018soft} agent called {\fontfamily{cmss}\selectfont Switcher} whose policy's action set consists of $2$ actions: 1) call the DEEPTHINK LLM 2) call the QUICK model.
    \newline
     $\bullet$ \textbf{Switching Control Encoder.} To enable  SAC to perform in a textual environment we introduce an encoder to process text information. It consists of a transformer to turn text into a matrix and an MLP to condense information.  

\section{Convergence and Optimality of LONDI} \label{sec:convergence}

 \textcolor{black}{We now show that the LONDI framework converges to the solution that maximises the {\fontfamily{cmss}\selectfont Switcher} agent's value function and therefore the performance of the system. With this, the {\fontfamily{cmss}\selectfont Switcher} agent learns to activate DEEPTHINK only at the set of states at which doing so improves the system performance. In particular: Theorem 1 shows LONDI learns the optimal solution for {\fontfamily{cmss}\selectfont Switcher} so that it activates DEEPTHINK only when it is profitable to do so over the horizon of the problem (recall that each activation incurs a DEEPTHINK cost).  Finally, we characterise the optimal DEEPTHINK activation points and show that {\fontfamily{cmss}\selectfont Switcher} can use a condition on its action-value function that can be evaluated online to determine when to activate DEEPTHINK (for the case when {\fontfamily{cmss}\selectfont Switcher} uses a Q-learning variant).  All our results are built under Assumptions 1 - 7 (Sec. \ref{sec:notation_appendix} of the Appendix) which are standard in RL and stochastic approximation theory \cite{bertsekas1996neuro}.
 }
 
 \textcolor{black}{The following theorem shows the solution of {\fontfamily{cmss}\selectfont Switcher}'s problem is a limit point of a sequence of Bellman operations acting on a value function (i). It then shows {\fontfamily{cmss}\selectfont Switcher} converges to the solution (ii). }

\begin{theorem}\label{theorem:existence}
\textbf{i)} Let $v_S:\mathcal{S}\to\mathbb{R}$ then for any fixed policies ${\pi}^{\rm QUICK},{\pi}\in {\Pi}$ the solution of {\fontfamily{cmss}\selectfont Switcher}'s problem is given by $\underset{k\to\infty}{\lim}T_S^kv_S(\cdot|{\pi},\mathfrak{g})=\underset{\hat{\mathfrak{g}}}{\max}\;v_S(\cdot|{\pi},\hat{\mathfrak{g}})$,\\
where $T_S$ is given by $
T_S v_S:=\max\Big\{\mathcal{M}^{\mathfrak{g},{\pi}^{\rm DEEP}}Q_S,\underset{{a}\in{\mathcal{A}}}{\max}\;\left[ R_S+\gamma\sum_{s'\in\mathcal{S}}P(s';\cdot)v_S(s')\right]\Big\}$ and  
$
\mathcal{M}^{\mathfrak{g},{\pi}^{\rm DEEP}}Q_S(s,{a}|\cdot):=Q_S(s,{\pi}^{\rm DEEP}(s)|\cdot)-c$ \textcolor{black}{which measures the expected return for {\fontfamily{cmss}\selectfont Switcher} following a switch to the DEEPTHINK model minus the switching cost $c$.}
\newline
\textbf{ii)} LONDI converges under a Q-learning variant.   

\end{theorem}
\textcolor{black}{Therefore, Theorem \ref{theorem:existence} proves the solution to {\fontfamily{cmss}\selectfont Switcher}'s problem in which {\fontfamily{cmss}\selectfont Switcher} optimally selects the set of states to activate DEEPTHINK can be obtained by computing the limit of a (RL) dynamic programming procedure  (when {\fontfamily{cmss}\selectfont Switcher} uses a Q-learning variant). Secondly, it proves the LONDI system converges to the solution.} It is easy to see that an immediate consequence of the theorem is that LONDI learns to make the minimum number of DEEPTHINK calls required to learn the solution to the {\fontfamily{cmss}\selectfont Switcher} problem since any additional DEEPTHINK calls would render the {\fontfamily{cmss}\selectfont Switcher} agent's policy suboptimal.

The following result characterises {\fontfamily{cmss}\selectfont Switcher}'s policy $\mathfrak{g}$:

\begin{proposition}\label{prop:switching_times}
For any $s_t\in\mathcal{S}$ and for all ${a}_t\in{\mathcal{A}}$, the policy $\mathfrak{g}$ is given by: $\mathfrak{g}(\cdot|s_t)={1}_{\mathbb{R}_+}\left(\mathcal{M}^{\mathfrak{g},{\pi}^{\rm DEEP}}Q_S(s_t,{a}_t|\cdot)-\underset{{a}_t\in{\cA}}\max\; Q_S(s_t,{a}_t|{\pi},\mathfrak{g})\right)$ where ${1}$ is the indicator function.
\end{proposition}
 Prop. \ref{prop:switching_times} provides a characterisation of where {\fontfamily{cmss}\selectfont Switcher} should activate DEEPTHINK. The condition can be evaluated online during the training phase.

\section{LONDI with a DEEPTHINK Call Budget} \label{sec:londi_budget}
So far we have considered the case in which the aim of the {\fontfamily{cmss}\selectfont Switcher}  freely activates DEEPTHINK at the set of states that using DEEPTHINK leads to an appreciable increase in performance. The paramater $c$ plays a critical role in ensuring that the {\fontfamily{cmss}\selectfont Switcher} activates DEEPTHINK only at such states. Although the intuition behind the role that $c$ plays is clear, in various applications, the choice of $c$ is not obvious and maybe hard to finetune. In this section, we consider an alternative view and introduce a variant of LONDI, namely LONDI-B that imposes a budgetary constraint of the number of DEEPTHINK activation that can be executed by {\fontfamily{cmss}\selectfont Switcher}. This provides another facility trade-off between the resource-consumption vs performance to be calibrated according to the User's preference. 
We thereafter prove that by tracking its remaining budget, LONDI-B learns a policy that makes optimal usage of its DEEPTHINK budget while preserving the budget constraint on DEEPTHINK calls almost surely.  

In what follows, we now consider a problem setting in which the {\fontfamily{cmss}\selectfont Switcher} agent now faces a fixed budget on the number of DEEPTHINK calls. This leads to the following constrained MDP setting:
\begin{align*}
        \max\limits_{\mathfrak{g}}~& v_S(s|{\pi},\mathfrak{g})\;\; 
        \text{s. t. } n - \sum_{j<\infty }\sum_{t_j\geq 0}{1}(\mathfrak{g}(\cdot|s_{t_j}))\geq 0, 
\end{align*}
where $n\geq 0$ is a fixed value that represents the budget for the number of DEEPTHINK activations and $k=1,\ldots$ represents the training episode count.  As in \citep{sootla2022saute}, we introduce a new variable $x_t$ that tracks the remaining number of activations: $x_t := n - \sum_{t\geq 0}{1}(\mathfrak{g}(s_t))$ where the variable
$x_t$ is now treated as the new state variable which is a component in an augmented state space $\cX:=\cS\times\mathbb{N}$. We introduce the associated reward function $\widetilde{R}_S:\cX\times{\boldsymbol{\cA}}\to\cP(D)$ and the probability transition  function $\widetilde{P}:\cX\times{\cA}\times\cX\to[0,1]$ whose state space input is now replaced by $\cX$ and the {\fontfamily{cmss}\selectfont Switcher} value function for the MDP $\langle \cS,\left(\cA,\cA_S\right),\tilde{P},\tilde{R}_S,\gamma\rangle$. We now prove LONDI-B ensures maximal performance for a given DEEPTHINK call budget.

\begin{theorem} \label{thm:optimal_policy_budget} In 
the budgeted  problem,
for any $\widetilde{v}:\cX\to \mathbb{R}$, the solution of $\widetilde{\cG}$ is given by $
\underset{k\to\infty}{\lim}\tilde{T}_S^k\widetilde{v}^{{\pi}}=\underset{\mathfrak{g}}{\max}\;\widetilde v^{{\pi},\mathfrak{g}}$, where {\fontfamily{cmss}\selectfont Switcher}'s optimal policy takes the Markovian form $\widetilde{\mathfrak{g}}(\cdot | y)$ for any $y\in\cX$. 
\end{theorem}
Theorem \ref{thm:optimal_policy_budget}  shows LONDI converges under standard assumptions to the solution of {\fontfamily{cmss}\selectfont Switcher}'s problem when {\fontfamily{cmss}\selectfont Switcher} faces a DEEPTHINK call budget constraint.
\section{Algorithm}
We now describe the methodology for the LONDI framework. At a state $s_t$, the system initially checks the switch state $m_t$. If the switch is currently off, $m_t=0$, the {\fontfamily{cmss}\selectfont Switcher} agent is applied to determine whether DEEPTHINK is needed. Otherwise $m_t=1$, the structure considers the switching probability $p_i$\footnote{The probability here is a Bernoulli distribution $p(x)=p^x(1-p)^{1-x}$ and the parameter $p$ and $q=1-p$ of the distribution is adjustable based on the difficulty of the environment. Considering that the DEEPTHINL model in our setting (Flan model) cannot undertake long step planning, the switching probability $p_i$ is applied ensures the possibility of sustained activation of the DEEPTHINK model, which improves the performance in complex and sparse reward tasks. $p_i(\cdot|s_t)$ means the value of $p_i$ at state $s_t$, whose value is sampled from the Bernoulli distribution.}. If $p_i(\cdot|s_t)=1$, the switch keeps the same and the structure directly utilises DEEPTHINK to generate action. If $p_i(\cdot|s_t)=0$, the structure activates {\fontfamily{cmss}\selectfont Switcher} agent to decide whether to turn on the switch again. If the $\mathsf{Switcher}$ module is used, it considers the state information $s_t$ processed by the encoder and samples the action according to its discrete policy $g_t\sim\mathfrak{g}:\mathcal{S}\rightarrow{0,1}$. If the {\fontfamily{cmss}\selectfont Switcher} output 0, $g_t=0$, the switch is off, $m_t=0$, and the QUICK model is called to generate action and interact with the environment. Otherwise $g_t=1$, the switch is on, $m_t=1$, and the DEEPTHINK model is activated to generate action. The trajectories of the process are stored in a replay buffer for further training. To further explain the work flow of LONDI, a schematic is placed in Appendix \ref{sec:schematic}. In the budget-constrained variant LONDI-B, the key distinction lies in the replacement of the cost function with a parameter budget $n$, effectively limiting the computational resources allotted for the task. The switching control policy employed by LONDI-B takes into account the remaining computational budget, adjusting its decision-making process accordingly by: $g_{t}\sim\mathfrak{g}(\cdot|s_{t},n)$. If the switching control mechanism invokes operations that surpass the predetermined budget limitations, the reward in and after that step will be reduced by a budget penalty $n_p$ , discouraging such over-expenditure of resources. Instead of completely prohibiting the agent from accessing the DEEPTHINK model after exhausting budget, we apply this additional penalty $n_p$ approach because it allows the agent to explore and consider information from later stages. If a particular point turns out to be highly important despite incurring a negative reward initially, the agent can strategically allocate a portion of its budget to access that point during subsequent training iterations. By applying a penalty rather than an outright restriction, the agent retains the flexibility to reevaluate and potentially utilize information from points that may become more valuable in later stages of the process.
\begin{algorithm}[!ht]
    \caption{\textbf{L}anguage \textbf{O}ptimising \textbf{N}etwork \textbf{Di}stribution (LONDI)}
    \label{alg:LONDI}
    \renewcommand{\algorithmicrequire}{\textbf{Input:}}
    \renewcommand{\algorithmicensure}{\textbf{Output:}}
    \begin{algorithmic}
        \REQUIRE QUICK policy $\pi^{\rm QUICK}$, DEEPTHINK policy $\pi^{\rm DEEP}$, Switching Control Policy $g$, learning algorithm $\Delta^{g}$, experience buffer $\textbf{B}$, switching probability $p_{i}$, switch state $m_{t}$ 
        \ENSURE Optimised policy $g$    
        \STATE Initialise $g$, $p_{i}$, $m_{t}$
        \WHILE{not done}
            \FOR{$t = 1, T$}
                \STATE Given environment state $s_{t}$ evaluate $g_{t}\sim\mathfrak{g}(\cdot|s_{t})$
                \IF{$m_{t} > 0$}
                    \IF{$p_{i}(\cdot|s_{t}) = 1$}
                        \STATE Sample action $a_{t}$ using DEEPTHINK policy $\pi^{\rm DEEP}$ 
                    \ELSE
                        \IF{$g_{t} = 1$}
                            \STATE Sample action $a_{t}$ using DEEPTHINK policy $\pi^{\rm DEEP}$
                        \ELSE
                            \STATE Sample action $a_{t}$ using QUICK policy $\pi^{\rm QUICK}$, $m_{t} = 0$ 
                        \ENDIF
                    \ENDIF
                \ELSIF{$g_{t} = 1$}
                    \STATE Sample action $a_{t}$ using DEEPTHINK policy $\pi^{\rm DEEP}$, $m_{t} += 1$
                \ELSE
                    \STATE Sample action $a_{t}$ using QUICK policy $\pi^{\rm QUICK}$ 
                \ENDIF
                \STATE Apply $a_{t}$ to environment to obtain $s_{t+1}$, $\tau_{t+1}$ and $r_{t+1}$
                \STATE Store ($s_{t},$, $g_{t}$, $r_{t+1}$, $s_{t+1}$) in $\textbf{B}$ 
            \ENDFOR
            \FOR{Epochs and Batch numbers}
                \STATE Sample $\textbf{B}$ to obtain ($s_{t}$,$g_{t}$,$r_{t+1}$,$s_{t+1}$) and update $\mathfrak{g}$ with $\Delta^{g}$    
            \ENDFOR
        \ENDWHILE    
    \end{algorithmic}
\end{algorithm}
\begin{algorithm}[!ht]
    \caption{\textbf{L}anguage \textbf{O}ptimising \textbf{N}etwork \textbf{Di}stribution-\textbf{B}udget (LONDI-B)} 
    \label{alg:LONDI-B}
    \renewcommand{\algorithmicrequire}{\textbf{Input:}}
    \renewcommand{\algorithmicensure}{\textbf{Output:}}
    \begin{algorithmic}
        \REQUIRE QUICK policy $\pi^{\rm QUICK}$, DEEPTHINK policy $\pi^{\rm DEEP}$, Switching Control Policy $g$, learning algorithm $\Delta^{g}$, experience buffer $\textbf{B}$, switching probability $p_{i}$, switch state $m_{t}$, switch budget $n$, switch budget limit $n_0$, switch budget penalty $n_p$, switch cost $c$ 
        \ENSURE Optimised policy $g$    
        \STATE Initialise $g$, $p_{i}$
        \FOR{$t = 1, T$}
            \STATE $m_{t}=0, n=n_0$, 
            \WHILE{not done}
                \STATE Given environment state $s_{t}$ and switch budget $n$, evaluate $g_{t}\sim\mathfrak{g}(\cdot|s_{t},n)$
                \IF{$m_{t} > 0$}
                    \IF{$p_{i}(\cdot|s_{t}) = 1$}
                        \STATE Sample action $a_{t}$ using DEEPTHINK policy $\pi^{\rm DEEP}$
                    \ELSE
                        \IF{$g_{t} = 1$}
                            \STATE Sample action $a_{t}$ using DEEPTHINK policy $\pi^{\rm DEEP}$, $n-=1$
                        \ELSE
                            \STATE Sample action $a_{t}$ using QUICK policy $\pi^{\rm QUICK}$, $m_{t} = 0$
                        \ENDIF
                    \ENDIF
                \ELSIF{$g_{t} = 1$}
                    \STATE Sample action $a_{t}$ using DEEPTHINK policy $\pi^{\rm DEEP}$, $m_{t} += 1$, $n-=1$
                \ELSE
                    \STATE Sample action $a_{t}$ using QUICK policy $\pi^{\rm QUICK}$ 
                \ENDIF
                \STATE Apply $a_{t}$ to environment to obtain $s_{t+1}$, $\tau_{t+1}$ and $r_{t+1}$
                \IF{$g_{t} = 1$}
                    \STATE $r_{t+1}-=c$
                \ENDIF
                \IF{$n<0$}
                    \STATE $r_{t+1}-=n_p$
                \ENDIF
                \STATE Store ($s_{t},$, $g_{t}$, $r_{t+1}$, $s_{t+1}$,$n$) in $\textbf{B}$ 
            \ENDWHILE    
            \FOR{Epochs and Batch numbers}
                \STATE Sample $\textbf{B}$ to obtain ($s_{t}$,$g_{t}$,$r_{t+1}$,$s_{t+1}$,$n$) and update $\mathfrak{g}$ with $\Delta^{g}$   
            \ENDFOR
        \ENDFOR
    \end{algorithmic}
\end{algorithm}
\section{Experiments}\label{sec:exp_main}
We conducted a series of experiments to test whether LONDI: \textbf{1.} Solves complex interactive tasks while reducing the computational cost. \textbf{2.} Works within the budget limit. \textbf{3.} Develops the ability to optimise its utilisation of DEEPTHINK within a specified budget for DEEPTHINK calls. \textbf{4.} Is plug \& play, namely robust and consistent with different components. 
We use ScienceWorld \cite{wang2022scienceworld} and BabyAI-Text \cite{carta2023grounding} environments to test the performance of LONDI.  We employ SAC \citep{haarnoja2018soft} to learn the control policy for switching. The presented plots display the average results obtained from $5$ different seeds \footnote{The standard deviation across different seeds are shown in the tables and charts.}. 
%
%

We evaluated the effectiveness of LONDI in two environments, namely 
Scienceworld and BabyAI-Text. Additional information regarding experimental setups is shown in \ref{sec:imp details}. \hspace*{\fill} \\
\textbf{ScienceWorld.} The ScienceWorld environment \cite{wang2022scienceworld} simulates a residential setting comprising 10 interconnected areas with a diverse range of up to 200 objects, including devices, instruments, plants/animals, electrical components, substances, and containers, as well as common environmental items like furniture, books, and paintings. The action space in the ScienceWorld environment consists of 25 high-level actions, encompassing both science-specific actions and general actions. The agent can only observe the information of its current area. For different tasks, the agent needs to combine high-level actions and objects into applicable actions and receive periodic rewards if they move towards the goal. A sample illustration of one task: \textit{Create a circuit} is displayed in Figure \ref{fig_sw}.  
\\
\begin{wrapfigure}{r}{4cm}
    \centering
    \includegraphics[width=0.25\textwidth,height=0.17\textwidth]{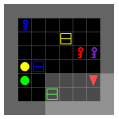}
    \caption{An illustration of one BabyAI task, PutNextLocal:"put the blue key next to the green ball". The shadow area represents the observable space of the agent.}
    \label{babyai_fig}
    \vspace{-.5cm}
\end{wrapfigure} 
{\textbf{BabyAI-Text.}  BabyAI \cite{chevalierboisvert2019babyai} is a 2D grid-based simulation environment that offers tasks of increasing complexity. 
The environment has various objects and the agent can pick up, drop, and move objects, and doors can be unlocked using keys of matching colors, which may be concealed within boxes. The agent's field of vision is limited to a 7x7 grid, and it cannot see beyond walls or closed doors. The available actions for the agent include moving forward, turning left or right, opening doors or boxes, picking up items, dropping items, and signaling completion. The agent can only hold one item at a time. The objective is to reach the goal state as quickly as possible, with the goal state being assigned a reward that diminishes over time. In this experiment, we use the modified textual version of BabyAI proposed by \cite{carta2023grounding}. One example task is shown in Figure \ref{babyai_fig}. 


We use a pre-trained Flan-T5 small model as our QUICK model and a pre-trained Flan-T5 large model \citep{chung2022scaling} as our DEEPTHINK model in ScienceWorld environment. For BabyAI-Text, We apply GFlan small and large \citep{carta2023grounding} as our QUICK and DEEPTHINK model, which builds upon the LLM Flan-T5 model as the basis for its action policy and optimises it using the PPO algorithm. The baselines are SWIFTSAGE, FrugalGPT and Probabilistic policy. The details of the baselines are shown in Appendix \ref{sec:baselines}. Ablation studies are displayed in Appendix \ref{sec:abl_results} to verify LONDI is robust to different components and the effectiveness of switching control structure. 

\subsection{Results and Analysis} \label{results_main}
\textbf{Switching cost variation.} To evaluate the performance of LONDI, we verified it with different cost values on different benchmarks. The results indicate that the cost $c$ enables the resource/performance trade-off to be carefully calibrated by LONDI. With small cost, LONDI is willing to activate DEEPTHINK more and achieve higher reward. But when the cost is expensive, LONDI becomes extremely economical and rarely use DEEPTHINK model. Therefore, LONDI in this situation uses QUICK model mostly and the reward is reduced. The results shown in the Tables \ref{tab:dif_cos_sw} and \ref{tab:dif_cos_babyai} suggest that as the computational cost increases, LONDI's performance declines, implying that higher costs lead to fewer invocations of the resource-intensive DEEPTHINK model.
\begin{minipage}{\textwidth}
\begin{minipage}[t]{0.48\textwidth}
\makeatletter\def\@captype{table}
\begin{tabular}{cc}
    \toprule
    Model     & Reward      \\
    \midrule
    DeepTHINK only & 77.6 $\pm$ 2.3    \\
    LONDI (cost=0.1)     & 76.5$\pm$3.4     \\
    LONDI (cost=0.2)     & 73.3$\pm$3.7       \\
    LONDI (cost=0.3)     & 49.6$\pm$2.6       \\
    LONDI (cost=0.4)     & 43.3$\pm$2.0      \\
    QUICK only    & 43.2$\pm$1.7        \\
    \bottomrule
\end{tabular}
\caption{Performance of LONDI on ScienceWorld task: \textit{Identify Longest-then-shortest-lived animal} with different cost (normalized)}
\label{tab:dif_cos_sw}
\end{minipage}
\begin{minipage}[t]{0.48\textwidth}
\makeatletter\def\@captype{table}
\begin{tabular}{ccc}
\toprule  
Model&Reward&Success Rate\\
\midrule  
DEEPTHINK only&0.96$\pm$0.02&0.87$\pm$0.05\\
LONDI (cost=0.05)&0.91$\pm$0.03&0.79$\pm$0.05\\
LONDI (cost=0.15)&0.82$\pm$0.05&0.72$\pm$0.07\\
LONDI(cost=0.25)&0.69$\pm$0.05&0.51 $\pm$0.08\\
LONDI (cost=0.35)&0.51$\pm$0.03&0.36$\pm$0.04\\
QUICK only&0.49$\pm$0.01&0.34$\pm$0.02\\
\bottomrule 
\end{tabular}
\caption{Performance of LONDI on BabyAI-Text with mixed tasks. The cost and reward are normalised values.}
\label{tab:dif_cos_babyai}
\end{minipage}
\end{minipage}

\textbf{Budget version of LONDI.} LONDI-B, the budget version of LONDI, possesses a more straightforward adjustable parameter in comparison. This parameter directly constrains the utilization of the DEEPTHINK model, thereby limiting the computational expenses of the module. To assess the computational resource requirements, we measure the average performance of the budget version LONDI-B across five evaluation episodes and monitor the GPU usage. The areas under the GPU usage curve are calculated based on numpy function. The results are shown in Table \ref{cost_main}. The DEEPTHINK calls in all LONDI-B structures remain within the specified budget. With a higher budget, the DEEPTHINK module is activated more frequently and the performance increases correspondingly. Note that when the budget is set to 2, the performance of LONDI is comparable to the no-limit version (DEEPTHINK only), whereas the DEEPTHINK activation is approximately two-fifths compared to the DEEPTHINK only version. The computation metrics AUC indicate that LONDI-B significantly decreases the computational resource usage compared with only using the large language model. Specifically, LONDI with a budget of achieves 90$\%$ performance compared with DEEPTHINK only but calls the DEEPTHINK model almost 60$\%$ fewer times.
\begin{table}[!htbp]\vspace{-.5cm}
\centering
\caption{Computational cost of LONDI-B on the ScienceWorld task: \textit{Identify Longest-then-shortest-lived animal}. DEEPTHINK Calls column represents the relative percentage of DEEPTHINK activations compared with DEEPTHINK only. AUC column gives the area under the GPU usage curve for the same number of timesteps.}
\hspace{-.4cm}\begin{tabular}{cccc}
\toprule  
Model&DEEPTHINK Calls&Reward& AUC\\
\midrule  
DEEPTHINK only&5$\pm$0.00(1)&77.6$\pm$2.3& 3130$\pm$5\\
LONDI-B (budget=5)&4.87$\pm$0.04 (0.97)&76.5$\pm$3.3& 2968$\pm$45\\
LONDI-B (budget=4)&3.96$\pm$0.03 (0.79)&75.3$\pm$3.5& 2843$\pm$53\\
LONDI-B (budget=3)&2.82$\pm$0.04 (0.56)&72.3$\pm$3.7& 2759$\pm$47\\
LONDI-B (budget=2)&1.72$\pm$0.05 (0.34)&70.6$\pm$3.2& 2671$\pm$43\\
LONDI-B (budget=1)&0.83$\pm$0.02 (0.17)&65.7$\pm$2.1& 2593$\pm$32\\
QUICK only&0$\pm$0.00&43.3$\pm$1.7& 2451$\pm$2\\
\bottomrule 
\end{tabular}
\label{cost_main}
\end{table}
\newline
\textbf{Comparison to Baselines.} To compare the performance of LONDI-B with baselines, we modify the QUICK and DEEPTHINK modules of SWIFTSAGE into the same setting, namely Flan-T5 small and large. The budget for calling DEEPTHINK model is set to be half of the calls of DEEPTHINK only in all tasks for LONDI-B and baselines. The results shown in Table \ref{tab:baseline_compare} indicate that LONDI-B's performance is marginally inferior to DEEPTHINK only, but outperforms SWIFTSAGE, FrugalGPT and probabilistic policy in both baselines. In addition, considering that LONDI-B only calls the DEEPTHINK model half than DEEPTHINK only, it significantly minimize the computational cost of the system. More detailed information can be found in Appendix \ref{tab:t2}. \footnote{Because the switching policy employed by SWIFTSAGE is based on predefined rules and empirical observations specific to the ScienceWorld environment, it becomes challenging for this approach to perform effectively when applied to other environments or domains that may have different dynamics and characteristics.} 
\begin{table}[!htbp]\vspace{-.5cm}
\centering
\caption{Average performance of LONDI-B and baselines on ScienceWorld and BabyAI mixed tasks. The reward value is normalized.}
\hspace{-.4cm}\begin{tabular}{ccc}
\toprule  
Model&ScienceWorld&BabyAI\\
\midrule  
DEEPTHINK only&0.87$\pm$0.02&0.96$\pm$0.02\\
LONDI-B &0.72$\pm$0.06&0.87$\pm$0.05\\
SWIFTSAGE &0.68$\pm$0.04&0.43$\pm$0.02\\
FrugalGPT&0.64$\pm$0.03&0.75$\pm$0.03\\
Probabilistic policy&0.44$\pm$0.08&0.41$\pm$0.07\\
QUICK only&0.33$\pm$0.01&0.34$\pm$0.01\\
\bottomrule 
\end{tabular}
\label{tab:baseline_compare}
\end{table}

\section{Conclusion}
In this paper, we introduce LONDI, a novel framework that leverages performance and computational cost by selectively activating LLM to cooperate with LM. LONDI combines LM and LLM in a way that enables LLM to support LM, thereby improving its performance. Simultaneously, LONDI can assist LLM in reducing computational and energy consumption. The budget variant of LONDI, known as LONDI-B, enhances the combination by offering an intuitive control facility for the user to limit the number of LLM calls . Through the translation of the switching problem into a Markov Decision Process, our theoretical framework demonstrates that LONDI maintains the convergence guarantees of RL while enhancing overall performance. In our empirical investigations, we conducted a comprehensive set of experiments in ScienceWorld and BabyAI-Text. Across these domains, LONDI shows performance improvements and computational cost decreases. 
\normalem 
\bibliography{main}
\bibliographystyle{abbrvnat}
\appendix
\onecolumn
{\centering \huge Appendix}
\section{Schematic of LONDI}\label{sec:schematic}
The schematic of LONDI is shown in Figure \ref{fig: flowdig}.
\begin{figure}[ht!]
    \centering
    \includegraphics[width=0.48
    \textwidth,height=0.55\textwidth]{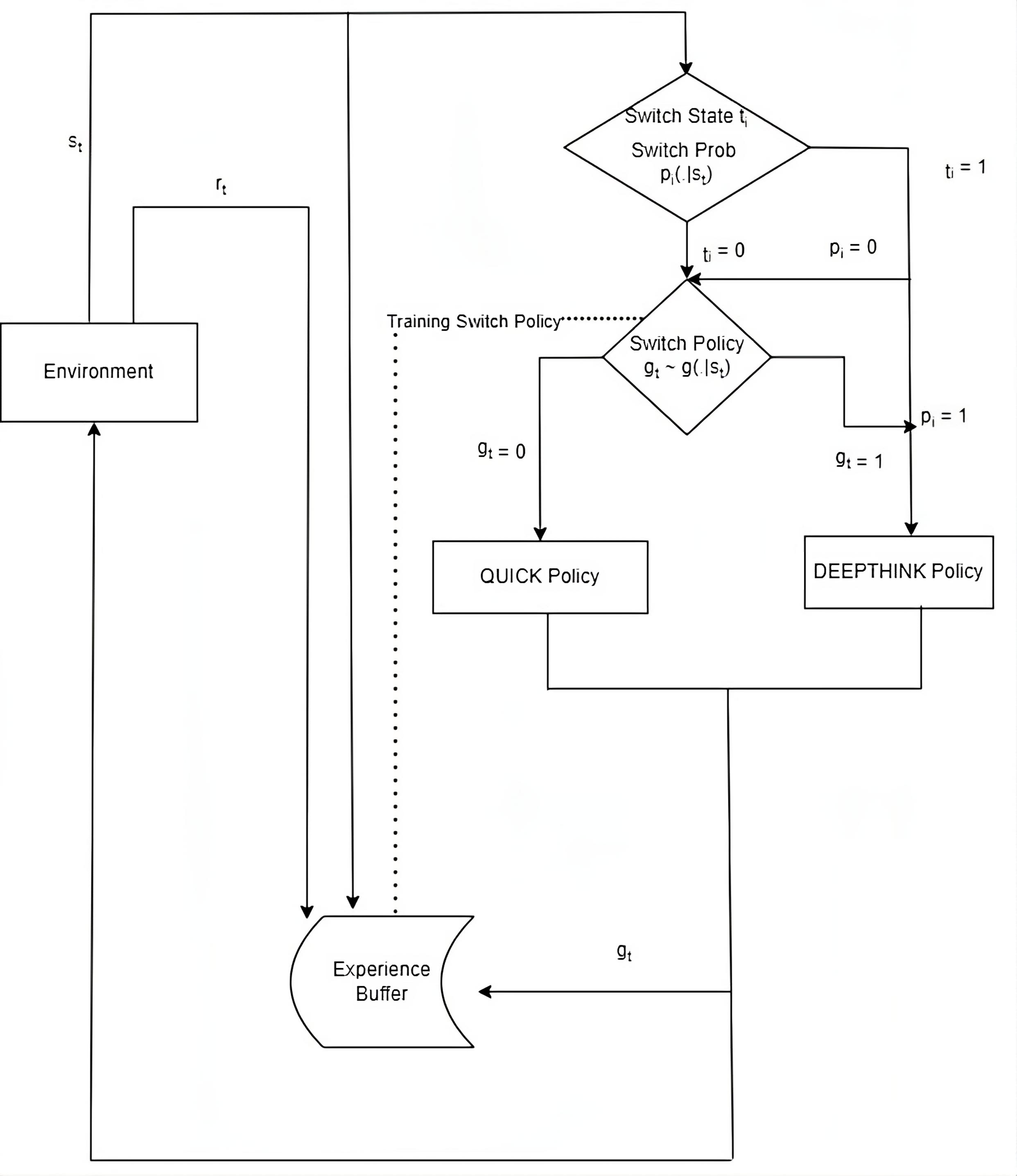}
    \caption{The schematic of LONDI. The Diamond represents the decision point, the square means a process or action, and the oval-like shape means data storage. The {\fontfamily{cmss}\selectfont Switcher} agent receives an environmental observation and makes a decision on which LLM module to utilise based on factors such as switch state, switch probability, and observation. The transition is then stored in the buffer for training the {\fontfamily{cmss}\selectfont Switcher} policy in subsequent iterations. 
    }
    \label{fig: flowdig}
    \vspace{-.8cm}
    \end{figure} 
\newline
\section{Ablation Study}\label{sec:abl_results}
We conducted a series of ablation studies to verify the effectiveness of LONDI. In the subsequent analysis, we made adjustments to various elements of our framework to substantiate the following assertions:\newline

\textbf{LONDI effectively adapts to different components.} To validate the dynamic modification capability of LONDI in activating the DEEPTHINK module based on the QUICK module's capability, we replaced the QUICK module with both a random agent and a more proficient FLAN-T5-small model. As shown in Figure \ref{random} and Table \ref{dif_swift}, the results indicate that LONDI learns to activate the DEEPTHINK module according to the QUICK's performance. With a random QUICK module and a lower-performance QUICK module, LONDI can still achieve similar performance by activating the DEEPTHINK module more frequently. With lower performance DEEPTHINK module Flan-T5-small, LONDI shows the same tendency with various costs, which indicates that LONDI is a plug-in structure that can be used with distinct LLMs by only adjusting a few hyperparameters. \newline
\begin{figure}[htbp]\vspace{-.1cm}
    \centering
    \includegraphics[width=0.4\textwidth,height=0.23\textwidth]{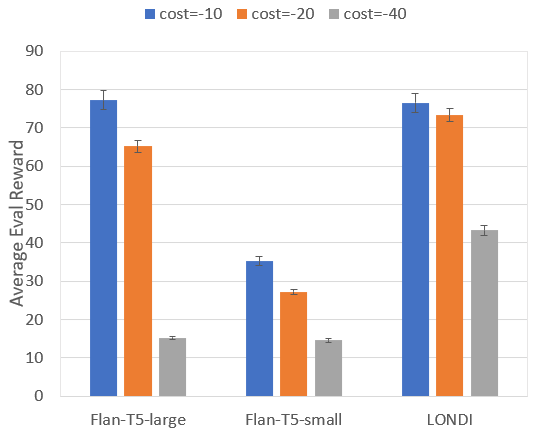}
    \caption{Performance of LONDI with a random agent as QUICK module on ScienceWorld task: \textit{Identify Longest-then-shortest-lived animal}. The Flan-T5-large bar represents Flan-T5-large as the DEEPTHINK module. The Flan-T5-small bar represents Flan-T5-small as the QUICK module.}
    \label{random}
    \vspace{-.7cm}
    \end{figure} 
\newline

\begin{table}[!htbp]
\caption{Performance of LONDI with different QUICK modules on ScienceWorld task: \textit{Create a circuit}. LONDI(L) means we modify the QUICK module to another FLAN-T5-small model which has a longer training length on this task.}
\vspace{3mm}
\centering\begin{tabular}{ccc}
\toprule  
Model& Reward& Rel. DEEPTHINK Calls\\
\midrule  
DEEPTHINK only& 70$\pm$2.0& 1$\pm$0.00\\
LONDI(L)& 56$\pm$3.5& 0.3$\pm$0.02\\
LONDI& 52$\pm$3.1& 0.5$\pm$0.04\\
QUICK only& 9$\pm$0.0& 0$\pm$0.00\\
FLAN-T5-small-L& 30$\pm$2.3& 0$\pm$0.00\\
\bottomrule 
\end{tabular}
\label{dif_swift}
\end{table}
\textbf{Switching Controls are important.}
A central component of LONDI is its switching control mechanism which determines when to activate the DEEPTHINK model. In particular, the switching control mechanism allows the {\fontfamily{cmss}\selectfont Switcher} agent to learn to activate the DEEPTHINK model only at states where it is needed to drive higher performance. To evaluate the importance of the switching control component and the effect of budget constrain, we compared the performance of LONDI-B with a variant of LONDI-B in which the switching control mechanism is replaced with probabilistic policy. Observe that activating the DEEPTHINK model at all states degenerates the method to DEEPTHINK and similarly, never activating the DEEPTHINK degenerates the framework to QUICK. Table \ref{dif_budget_random} shows the comparison of the performances of the variants under different budgets. We examined the performance of the variants of LONDI-B on the task called identifying the longest-lived animal. The results indicate that having the switching control component and hence, the ability to learn an optimal switching control in LONDI-B produces a significantly better performance compared to simply activating DEEPTHINK with Bernoulli distribution while LONDI-B uses fewer DEEPTHINK calls.
\begin{table}[!htbp]
\caption{The performance of LONDI-B and a variant of LONDI-B with a random {\fontfamily{cmss}\selectfont Switcher} (Rand. {\fontfamily{cmss}\selectfont Switcher}) agent with different budgets on ScienceWorld task: \textit{Identify Longest-lived animal}. Data in blue and brown are related to LONDI-B and the random variant LONDI-B resp. The budget usage column represents the DEEPTHINK calls of LONDI under the different budget settings, where the budget usage of variant LONDI-B is always equal to the setting number. The average row represents the mean value of structures whose budget greater than zero. LONDI-B outperforms the variant with a random agent for all budgets larger than zero.}
\vspace{3mm}
\centering
\begin{tabular}{cccc}
\toprule  
Budget&\textcolor{blue}{LONDI-B}&\textcolor{brown}{Probabilistic policy}&Budget Usage\\
\midrule  
No limit&\textcolor{blue}{83.2$\pm$2.1}&\textcolor{brown}{69.3$\pm$7.2}&\textcolor{blue}{3.93$\pm$0.04}$\backslash$\textcolor{brown}{4.8$\pm$0.42}\\
budget=4&\textcolor{blue}{81.6$\pm$4.2}&\textcolor{brown}{68.2$\pm$6.9}&\textcolor{blue}{3.82$\pm$0.06}$\backslash$\textcolor{brown}{4$\pm$0.00}\\
budget=3&\textcolor{blue}{77.6$\pm$5.3}&\textcolor{brown}{66.4$\pm$7.5}&\textcolor{blue}{2.76$\pm$0.07}$\backslash$\textcolor{brown}{3$\pm$0.00}\\
budget=2&\textcolor{blue}{75.8$\pm$6.1}&\textcolor{brown}{59.7$\pm$8.1}&\textcolor{blue}{1.87$\pm$0.09}$\backslash$\textcolor{brown}{2$\pm$0.00}\\
budget=1&\textcolor{blue}{66.4$\pm$3.2}&\textcolor{brown}{50.8$\pm$6.4}&\textcolor{blue}{0.82$\pm$0.04}$\backslash$\textcolor{brown}{1$\pm$0.00}\\
budget=0&\textcolor{blue}{42.3$\pm$1.6}&\textcolor{brown}{42.3$\pm$1.6}&\textcolor{blue}{0.02$\pm$0.01}$\backslash$\textcolor{brown}{0$\pm$0.00}\\
\textbf{Average}&\textcolor{blue}{\textbf{75.4$\pm$4.1}}&\textcolor{brown}{\textbf{61.2$\pm$6.1}}&\textcolor{blue}{\textbf{2.32$\pm$0.05}}$\backslash$\textcolor{brown}{\textbf{2.5$\pm$0.00}}\\
\bottomrule 
\end{tabular}
\label{dif_budget_random}
\vspace{-.5cm}
\end{table}
\section{Additional Results}
The detail results of all tasks \footnote{solvable by Flan-t5-large model} on the Scienceworld are shown in Table \ref{tab:t2}. The cost of LONDI and QUICK is the proportion compared to the AUV of the DEEPTHINK model in the same timesteps. The results indicated that LONDI facilitates the collaboration between QUICK and DEEPTHINK models. The structure performs slightly worse than DEEPTHINK but significantly better than QUICK. However, LONDI requires only slightly more computational resources than QUICK only, while consuming significantly less than DEEPTHINK only.
\begin{table}[!htbp]
\caption{Detailed results on the ScienceWorld benchmark across different tasks. The budget or allowance for invoking the computationally intensive DEEPTHINK model is set to be half of the number of times the DEEPTHINK model would be called if it were the sole model employed.}
\vspace{3mm}
\label{tab:t2}
\centering
\resizebox{\textwidth}{!}{\begin{tabular}{lcccccc}  
\hline
Task name&LONDI-B&QUICK&DEEPTHINK&LONDI-B cost&QUICK cost&DEEPTHINK cost\\ \hline
Find an animal&78$\pm$4.1&23$\pm$0.4&100$\pm$0.0&0.74$\pm$0.02&0.68$\pm$0.01&1\\ \hline
Find a living thing&75$\pm$2.9&20$\pm$0.2&100$\pm$0.0&0.76$\pm$0.01&0.7$\pm$0.01&1\\ \hline
Find a non-living thing&78$\pm$3.2&58$\pm$1.7&100$\pm$0.0&0.72$\pm$0.02&0.69$\pm$0.00&1\\ \hline
Find plant&89$\pm$3.1&34$\pm$1.2&100$\pm$0.0&0.78$\pm$0.03&0.64$\pm$0.00&1\\ \hline
Inclined planes(determine angle)&52$\pm$3.2&10$\pm$0.0&73$\pm$2.1&0.80$\pm$0.02&0.74$\pm$0.01&1\\ \hline
Friction(known surfaces)&64$\pm$4.1&38$\pm$1.4&73$\pm$2.2&0.85$\pm$0.02&0.65$\pm$0.01&1\\ \hline
Identify Longest-then-shortest-lived animal&72$\pm$3.7&43$\pm$1.6&78$\pm$2.3&0.85$\pm$0.03&0.78$\pm$0.00&1\\ \hline
Identify Longest-lived animal&76$\pm$6.1&42$\pm$1.6&83$\pm$2.1&0.79$\pm$0.02&0.72$\pm$0.00&1\\ \hline
Identify shortest-lived animal&87$\pm$3.2&50$\pm$1.8&100$\pm$0.0&0.83$\pm$0.02&0.74$\pm$0.01&1\\ \hline
Create a circuit&52$\pm$3.1&9$\pm$0.0&70$\pm$2.5&0.79$\pm$0.02&0.62$\pm$0.01&1\\ \hline
\end{tabular}} 
\end{table}
\section{Implementation Details}\label{sec:imp details}
\subsection{Hyperparameter Settings}
All hyperparameters used in our experiments are shown in the table below. The values included in square brackets indicate ranges of values that were used for performance tuning. All the training and evaluation is done on one NVIDIA A10 with 24GB memory. The training of LONDI with any version takes 8 hours under the setting of Table \ref{tab:hyper-param}. 
\begin{table}[!htbp]
\caption{Hyperparameter Setting of LONDI}
\vspace{3mm}
\label{tab:hyper-param}
\centering
\begin{tabular}{l|r}  
\hline
Clip Gradient Norm&1\\
$\gamma$&0.99\\
Learning rate&\num{1e-4}\\
Number of minibatches&4\\
Rollout length&128\\
Number of optimisation epochs&4\\
Optimisation algorithm&Adam\\
$\tau$&\num{5e-3}\\
$\epsilon$&\num{1e-8}\\
Encoder MLP layer&1\\
Encoder MLP hidden unit&256\\
Use Generalised Advantage Estimation&True\\
\hline
Coefficient of switch cost&[-1,-5,-10,-15,-20,-25,-30,-40]\\
Switch budget penalty&[-25,-45,-65]\\
Encoder output size&[4,8,16]\\
Switching probability&[0.1,0,3,0.5,0.7,0.9] \\
\hline
\end{tabular} 
\end{table}
\subsection{Training Details of Flan models}
Following the training setup of \cite{lin2023swiftsage} in ScienceWorld, we utilize flan-t5-large (783m) and flan-t5-small (77m) as the foundation, and fine-tuned them using the seq2seq action-prediction data (62k). In order to mitigate the potential bias arising from data imbalance in the sequence-to-sequence learning process, we employed a down-sampling technique, selectively reducing the representation of certain task types and actions, thereby curating a more balanced final dataset for the training phase. The training configs are consistent with \cite{lin2023swiftsage},  a learning rate of 1e-4 and batch size of 128 employed for training 500 steps. The detailed information of the dataset is shown in Figure \ref{dataset}.
\begin{figure}[htbp]\vspace{-.1cm}
    \centering
    \includegraphics[width=0.6\textwidth,height=0.6\textwidth]{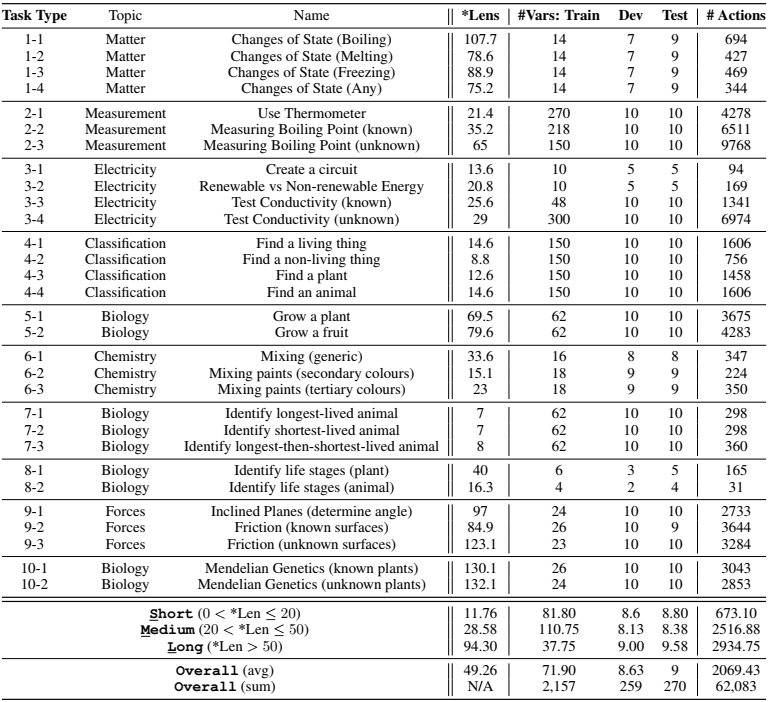}
    \caption{The statistics of ScienceWorld benchmark. \textit{*Len} represents the average length of the trajectories of oracle agents. Number of downsample variations are shown in each split. The rightmost column represents the quantity of data points utilized for the action-prediction seq2seq task during the training phase of Flan model.}
    \label{dataset}
    \vspace{-.7cm}
    \end{figure} 
\newline
\newline
For the training in BabyAI-Text, we follow the training framework of \cite{carta2023grounding}, which utilizes a Python library \textit{Lamorel} to enable the dispatching of calls to the deployed LLMs from a single line of code within the RL loop, requesting the probability of actions for all environments. Therefore, RL can call and communicate with all LLMs in parallel. In addition, \textit{Lamorel} help the update of LLMs using RL algorithm (i.e. PPO) loss. The hyperparameters of PPO are shown below. 
\begin{table}[!htbp]
\caption{Hyperparameter Setting of PPO in GFlan}
\vspace{3mm}
\label{tab:hyper-param-ppo}
\centering
\begin{tabular}{l|r}  
\hline
\textbf{Variables}&\textbf{Values}\\
\hline
Number of transitions collected between two updates &320 (8 environments ×40 steps in each environment)\\
Number of epochs per update&4\\
Batch size&32\\
Entropy loss coefficient&0.01\\
Value function loss coefficient&0.5\\
Discount factor&0.99\\
Learning rate&\num{1e-6}\\
$\lambda$ factor of the Generalized Advantage Estimator&0.99\\
Optimisation algorithm&Adam\\
Clipping parameter $\epsilon$&0.2\\
Maximum gradient norm&0.5\\
\hline
\end{tabular} 
\end{table}
\subsection{Prompt of DEEPTHINK model}
The prompt here is used by DEEPTHINK in ScienceWorld, the structure follows the setpup of \cite{lin2023swiftsage}, which has planning stage and grounding stage. The planning stage includes a concise summary of the task description and the sequence of previous actions, followed by posing five critical questions related to the current state. 
\newline
$\bullet$ “To complete the task, which objects do I need to collect? Please list them and their possible locations one by one.”
\newline
$\bullet$ “Are there any objects that have not been collected yet?”
\newline
$\bullet$ “To complete the task most efficiently, what are the important subgoals to achieve? Please list the subgoals one by one.”
\newline
$\bullet$  “Considering these subgoals, what have I already completed? And which subgoal should I focus on right now?”
\newline
$\bullet$ “Have I made any mistakes that might prevent me from efficiently completing the next subgoal? If any, how should I fix them?”
The grounding stage includes a comprehensive list of supported action types in a formal manner at first, complemented by remarks. The output of the planning stage are then given as advice. Recent action history of the past 10 time steps are also given. With these information, the LLM is required to focus on the next subgoal and generate a list of actions to achieve it. 
\subsection{Baselines}\label{sec:baselines}
We introduced all baseline methods and their experimental settings when we used to compare them
with the LONDI as follows.

\textbf{SWIFTSAGE}\cite{lin2023swiftsage} utilizes a rule-based method to switch the activation between DEEPTHINK and QUICK model: 1) There are five consecutive time steps with zero reward. 2) The QUICK’s prediction for the next action can result in a critical decision, such as giving the final answer for the experiment result.
3)  The QUICK’s prediction for the next action is invalid in the current environment or the observation of the action suggests that an exception is encountered. In this paper, we keep the DEEPTHINK and QUICK models of SWIFTSAGE consistent with those used in LONDI.

\textbf{FrugalGPT}\cite{chen2023frugalgpt} utilizes a cascade of LLMs and a score function to minimize the cost. The cascade tries with the smallest LLM first and the score function grades the response and query pair. If the score is higher than the threshold, the agent accept the answer. Otherwise, the cascade tries the larger LLM and the score function grades again until the answer meets the requirement or the largest LLM is used. In this experiment, we use the environment feedback as the score and set the threshold based on different tasks. In addition to the reward of environment, we added another reward that is proportional to the distance between the current location of agent and goal position. This was done to ensure more accurate feedback in alignment with the agent's actions. In this paper, the LLM cascade of FrugalGPT includes QUICK and DEEPTHINK models which are identical to those employed in LONDI.

\textbf{Probabilistic policy} is a policy that based on Bernoulli distribution \cite{BernoulliDistribution}: $P(x)=p^x(1-p)^{1-x}$. For the experiments, we configured the probability $p$ to be 0.5, implying that the values 0 and 1 had an equal chance of being selected. When the sample value of the distribution is 1, the policy decides to use the DEEPTHINK model. Otherwise, the policy use the QUICK model.  

\section{Limitations}\label{sec:limitation}
Although LONDI has shown adaptability to other environments, the trainable model free switcher limits its applicability to more resources constraint environments since the model-free algorithm requires large sample data for training which is not applicable in some real-world situation. One possible solution is to use model-based algorithms instead, which could alleviate sample efficiency problem. In addition, LONDI cannot be directly generalized to another unseen environments without training, which could reduce its application scope in real life. One possible approach is to combine reinforcement learning with causal inference to achieve better generalizability. 

\section{Impact Statement}\label{sec:impact state}
The objective of this paper is to reduce the overall computational consumption of LLM in addressing complex interactive tasks. One possible application of our work is to assist the downstream of LLM to edge devices, whose computational resources and energy consumption are limited. With LONDI, the edge devices can apply an affordable small LLM locally as QUICK and consider the cloud server as DEEPTHINK. Therefore, the edge devices can communicate to the cloud server for support only at necessary states, which minimises the overall consumption of energy and computational resources. This approach enables edge devices to achieve improved real-time decision-making capabilities, ensuring bandwidth and energy efficiency. Additionally, it could enables individuals to have an enhanced AI experience that is accessible to all with edge devices, ensures privacy through local data processing, and allows for customization and personalisation. 
\newpage
\section{Notation \& Assumptions}\label{sec:notation_appendix}

We assume that $\mathcal{S}$ is defined on a probability space $(\Omega,\mathcal{F},\mathbb{P})$ and any $s\in\mathcal{S}$ is measurable with respect
to the Borel $\sigma$-algebra associated with $\mathbb{R}^p$. We denote the $\sigma$-algebra of events generated by $\{s_t\}_{t\geq 0}$
by $\mathcal{F}_t\subset \mathcal{F}$. In what follows, we denote by $\left( \cY,\|\|\right)$ any finite normed vector space and by $\mathcal{H}$ the set of all measurable functions.  Where it will not cause confusion (and with a minor abuse of notation) for a given function $h$ we use the shorthand $h^{(\pi^{i},\pi^{-i})}(s)= h(s,\pi^i,\pi^{-i})\equiv\mathbb{E}_{\pi^i,\pi^{-i}}[h(s,a^i,a^{-i})]$.

The results of the paper are built under the following assumptions which are standard within RL and stochastic approximation methods:

\textbf{Assumption 1}
The stochastic process governing the system dynamics is ergodic, that is the process is stationary, and every invariant random variable of $\{s_t\}_{t\geq 0}$ is equal to a constant with probability $1$.

\textbf{Assumption 2}
The {\fontfamily{cmss}\selectfont Switcher}'s reward function $R_S$ is in $L_2$.

\textbf{Assumption 3}
For any positive scalar $c$, there exists a scalar $\mu_c$ such that for all $s\in\mathcal{S}$ and for any $t\in\mathbb{N}$ we have: $
    \mathbb{E}\left[1+\|s_t\|^c|s_0=s\right]\leq \mu_c(1+\|s\|^c)$.

\textbf{Assumption 4}
There exists scalars $C_1$ and $c_1$ such that  $|R(s,\cdot)|\leq C_2(1+\|s\|^{c_2})$ for some scalars $c_2$ and $C_2$ we have that: $
    \sum_{t=0}^\infty\left|\mathbb{E}\left[R(s_t,\cdot)|s_0=s\right]-\mathbb{E}[R(s_0,\cdot)]\right|\leq C_1C_2(1+\|s_t\|^{c_1c_2})$.

\textbf{Assumption 5}
There exists scalars $e$ and $E$ such that for any $s\in\mathcal{S}$ we have that: $
    |R(s,\cdot)|\leq E(1+\|s\|^e)$ .

\textbf{Assumption 6}
For any {\fontfamily{cmss}\selectfont Switcher} policy $\mathfrak{g}$, the total number of interventions is $K<\infty$.

\section{Proof of Technical Results}\label{sec:proofs_appendix}

We begin the analysis with some preliminary results and definitions required to prove our main results.

\begin{definition}{A.1}
Given a norm $\|\cdot\|$, an operator $T: \cY\to \cY$ is a contraction if there exists some constant $c\in[0,1[$ for which for any $J_1,J_2\in  \cY$ the following bound holds: $    \|TJ_1-TJ_2\|\leq c\|J_1-J_2\|$.
\end{definition}

\begin{definition}{A.2}
An operator $T: \cY\to  \cY$ is non-expansive if $\forall J_1,J_2\in  \cY$ the following bound holds: $    \|TJ_1-TJ_2\|\leq \|J_1-J_2\|$.
\end{definition}

\begin{lemma} \cite{mguni2019cutting} \label{max_lemma}
For any
$f: \cY\to\mathbb{R}: \cY\to\mathbb{R}$, we have that the following inequality holds:
\begin{align}
\left\|\underset{a\in \cY}{\max}\:f(a)-\underset{a\in \cY}{\max}\: g(a)\right\| \leq \underset{a\in \cY}{\max}\: \left\|f(a)-g(a)\right\|.    \label{lemma_1_basic_max_ineq}
\end{align}
\end{lemma}

\begin{lemma} {A.4}\citep{tsitsiklis1999optimal}\label{non_expansive_P}
The probability transition kernel $P$ is non-expansive so that if $\forall J_1,J_2\in  \cY$ the following holds: $    \|PJ_1-PJ_2\|\leq \|J_1-J_2\|$.
\end{lemma} 

\section*{Proof of Theorem \ref{theorem:existence}}
To prove Theorem \ref{theorem:existence}, we follow the proof of the convergence theorem in MANSA \cite{pmlr-v202-mguni23a} while making adaptations to the single agent environment with a single learner.  
\begin{proof}
The proof of the Theorem proceeds by first proving that the  {\fontfamily{cmss}\selectfont Switcher} agent's learning process, which involves switching controls converges. Recall, that the {\fontfamily{cmss}\selectfont Switcher} agent presides over an activation that deactivates ${\pi}^{\rm QUICK}$ and activates ${\pi}^{\rm DEEP}$.



We begin by recalling the definition of the intervention operator $(\mathcal{M}^{\mathfrak{g},{\pi}^{\rm DEEP}}$ for any $s\in\cS$ and for a given ${\pi}$: 
\begin{align}
(\mathcal{M}^{\mathfrak{g},{\pi}^{\rm DEEP}}Q_S(s,{a}|\cdot):=Q_S(s,{\pi}(s)|\cdot)-c
\end{align}

Secondly, recall that the Bellman operator for the game $\cG$ is given by:

\begin{align}
T_Sv_S(s_{\tau_k}):=\max\left\{(\mathcal{M}^{\mathfrak{g},{\pi}^{\rm DEEP}}Q_S(s_{\tau_k},{a}),\underset{{a}\in{\mathcal{A}}}{\max}\;\left[ R_S(s_{\tau_k},{a},g)+\gamma\sum_{s'\in\mathcal{S}}P(s';{a},s_{\tau_k})v_S(s')\right]\right\}\label{bellman_proof_start}
\end{align}

To prove (i) it suffices to prove that $T$ is a contraction operator. Thereafter, we use both results to prove the existence of a fixed point for $\cG$ as a limit point of a sequence generated by successively applying the Bellman operator to a test value function.   
Therefore our next result shows that the following bounds hold:
\begin{lemma}\label{lemma:bellman_contraction}
The Bellman operator $T$ is a contraction so that the following bound holds: $
\left\|T\psi-T\psi'\right\|\leq \gamma\left\|\psi-\psi'\right\|$.
\end{lemma}


In the following proofs we use the following notation: $
\mathcal{P}^{{a}}_{ss'}=:\sum_{s'\in\mathcal{S}}P(s';{a},s)$ and $\mathcal{P}^{{\pi}}_{ss'}=:\sum_{{a}\in{\mathcal{A}}}{\pi}({a}|s)\mathcal{P}^{{a}}_{ss'}$.

To prove that $T$ is a contraction, we consider the three cases produced by \eqref{bellman_proof_start}, that is to say we prove the following statements:

i) $\qquad\qquad
\left| \underset{{a}\in{\mathcal{A}}}{\max}\;\left(R_S(s_t,{a},g)+\gamma\mathcal{P}^{{a}}_{s's_t}v_S(s')\right)-\underset{{a}\in{\mathcal{A}}}{\max}\;\left( R_S(s_t,{a},g)+\gamma\mathcal{P}^{{a}}_{s's_t}v_S'(s')\right)\right|\leq \gamma\left\|v_S-v_S'\right\|$

ii) $\qquad\qquad
\left\|(\mathcal{M}^{\mathfrak{g},{\pi}^{\rm DEEP}}Q_S-(\mathcal{M}^{\mathfrak{g},{\pi}^{\rm DEEP}}Q_S'\right\|\leq    \gamma\left\|v_S-v_S'\right\|,\qquad \qquad$.

iii) $\qquad\qquad
    \left\|(\mathcal{M}^{\mathfrak{g},{\pi}^{\rm DEEP}}Q_S-\underset{{a}\in{\mathcal{A}}}{\max}\;\left[ R_S(s_t,{a},g)+\gamma\mathcal{P}^{{a}}v_S'\right]\right\|\leq \gamma\left\|v_S-v_S'\right\|.
$

We begin by proving i).

Indeed, for any ${a}\in{\mathcal{A}}$ and $\forall s_t\in\mathcal{S}, \forall s'\in\mathcal{S}$ we have that 
\begin{align*}
&\left| \underset{{a}\in{\mathcal{A}}}{\max}\;\left(R_S(s_t,{a},g)+\gamma\mathcal{P}^\pi_{s's_t}v_S(s')\right)-\underset{{a}\in{\mathcal{A}}}{\max}\;\left( R_S(s_t,{a},g)+\gamma\mathcal{P}^{{a}}_{s's_t}v_S'(s')\right)\right|
\\&\leq \underset{{a}\in{\mathcal{A}}}{\max}\;\left|\gamma\mathcal{P}^{{a}}_{s's_t}v_S(s')-\gamma\mathcal{P}^{{a}}_{s's_t}v_S'(s')\right|
\\&\leq \gamma\left\|Pv_S-Pv_S'\right\|
\\&\leq \gamma\left\|v_S-v_S'\right\|,
\end{align*}
using the non-expaniveness of the operator $P$ and Lemma \ref{max_lemma}.

We now prove ii). Using the definition of $\mathcal{M}$ we have that for any $s_\tau\in\mathcal{S}$
\begin{align*}
&\left|((\mathcal{M}^{\mathfrak{g},{\pi}^{\rm DEEP}}Q_S-(\mathcal{M}^{\mathfrak{g},{\pi}^{\rm DEEP}}Q_S')(s_{\tau},{a}_{\tau})\right|
\\&=\Bigg|R_S(s_\tau,{\pi},g)- c+\gamma\mathcal{P}^{{\pi}}_{s's_\tau}\mathcal{P}^{{\pi}}v_S(s_{\tau})
-\left(R_S(s_\tau,{\pi},g)- c+\gamma\mathcal{P}^{{\pi}}_{s's_\tau}\mathcal{P}^{{\pi}}v_S'(s_{\tau})\right)\Bigg|
\\&\leq \underset{{a}_\tau,g\in {\mathcal{A}}\times \{0,1\}}{\max}    \Bigg|R_S(s_\tau,{a}_\tau,g)- c+\gamma\mathcal{P}^{{\pi}}_{s's_\tau}\mathcal{P}^{{a}}v_S(s_{\tau})
-\left(R_S(s_\tau,{a}_\tau,g)- c+\gamma\mathcal{P}^{{\pi}}_{s's_\tau}\mathcal{P}^{{a}}v_S'(s_{\tau})\right)\Bigg|
\\&= \gamma\underset{{a}_\tau,g\in {\mathcal{A}}\times \{0,1\}}{\max}    \Bigg|\mathcal{P}^{{\pi}}_{s's_\tau}\mathcal{P}^{{a}}v_S(s_{\tau})
-\mathcal{P}^{{\pi}}_{s's_\tau}\mathcal{P}^{{a}}v_S'(s_{\tau})\Bigg|
\\&\leq \gamma\left\|Pv_S-Pv_S'\right\|
\\&\leq \gamma\left\|v_S-v_S'\right\|,
\end{align*}
using the fact that $P$ is non-expansive. The result can then be deduced easily by applying max on both sides.

We now prove iii). We split the proof of the statement into two cases:

\textbf{Case 1:} 
First, assume that for any $s_\tau\in\cS$ and $\forall {a}\in{\cA}$ the following inequality holds:
\begin{align}(\mathcal{M}^{\mathfrak{g},{\pi}^{\rm DEEP}}Q_S(s_{\tau},{a})-\underset{{a}\in{\mathcal{A}}}{\max}\;\left(R_S(s_\tau,{a}_\tau,g)+\gamma\mathcal{P}^{{a}}_{s's_\tau}v_S'(s')\right)<0.
\end{align}

We now observe the following:
\begin{align*}
&(\mathcal{M}^{\mathfrak{g},{\pi}^{\rm DEEP}}Q_S(s_{\tau},{a})-\underset{{a}\in{\mathcal{A}}}{\max}\;\left(R_S(s_\tau,{a}_\tau,g)+\gamma\mathcal{P}^{{a}}_{s's_\tau}v_S'(s')\right)
\\&\leq\max\left\{\underset{{a}\in{\mathcal{A}}}{\max}\;\left(R_S(s_\tau,{a}_\tau,g)+\gamma\mathcal{P}^{{\pi}}_{s's_\tau}\mathcal{P}^{{a}}v_S(s')\right),(\mathcal{M}^{\mathfrak{g},{\pi}^{\rm DEEP}}Q_S(s_{\tau},{a})\right\}
-\underset{{a}\in{\mathcal{A}}}{\max}\;\left(R_S(s_\tau,{a}_\tau,g)+\gamma\mathcal{P}^{{a}}_{s's_\tau}v_S'(s')\right)
\\&\leq \Bigg|\max\left\{\underset{{a}\in{\mathcal{A}}}{\max}\;\left(R_S(s_\tau,{a}_\tau,g)+\gamma\mathcal{P}^{{\pi}}_{s's_\tau}\mathcal{P}^{{a}}v_S(s')\right),(\mathcal{M}^{\mathfrak{g},{\pi}^{\rm DEEP}}Q_S(s_{\tau},{a})\right\}
\\&\qquad-\max\left\{\underset{{a}\in{\mathcal{A}}}{\max}\;\left(R_S(s_\tau,{a}_\tau,g)+\gamma\mathcal{P}^{{a}}_{s's_\tau}v_S'(s')\right),(\mathcal{M}^{\mathfrak{g},{\pi}^{\rm DEEP}}Q_S(s_{\tau},{a})\right\}
\\&+\max\left\{\underset{{a}\in{\mathcal{A}}}{\max}\;\left(R_S(s_\tau,{a}_\tau,g)+\gamma\mathcal{P}^{{a}}_{s's_\tau}v_S'(s')\right),(\mathcal{M}^{\mathfrak{g},{\pi}^{\rm DEEP}}Q_S(s_{\tau},{a})\right\}-\underset{{a}\in{\mathcal{A}}}{\max}\;\left(R_S(s_\tau,{a}_\tau,g)+\gamma\mathcal{P}^{{a}}_{s's_\tau}v_S'(s')\right)\Bigg|
\\&\leq \Bigg|\max\left\{\underset{{a}\in{\mathcal{A}}}{\max}\;\left(R_S(s_\tau,{a}_\tau,g)+\gamma\mathcal{P}^{{a}}_{s's_\tau}v_S(s')\right),(\mathcal{M}^{\mathfrak{g},{\pi}^{\rm DEEP}}Q_S(s_{\tau},{a})\right\}
\\&\qquad-\max\left\{\underset{{a}\in{\mathcal{A}}}{\max}\;\left(R_S(s_\tau,{a}_\tau,g)+\gamma\mathcal{P}^{{a}}_{s's_\tau}v_S'(s')\right),(\mathcal{M}^{\mathfrak{g},{\pi}^{\rm DEEP}}Q_S(s_{\tau},{a})\right\}\Bigg|
\\&\qquad\qquad+\Bigg|\max\left\{\underset{{a}\in{\mathcal{A}}}{\max}\;\left(R_S(s_\tau,{a}_\tau,g)+\gamma\mathcal{P}^{{a}}_{s's_\tau}v_S'(s')\right),(\mathcal{M}^{\mathfrak{g},{\pi}^{\rm DEEP}}Q_S(s_{\tau},{a})\right\}-\underset{{a}\in{\mathcal{A}}}{\max}\;\left(R_S(s_\tau,{a}_\tau,g)+\gamma\mathcal{P}^{{a}}_{s's_\tau}v_S'(s')\right)\Bigg|
\\&\leq \gamma\underset{a\in\mathcal{A}}{\max}\;\left|\mathcal{P}^{{\pi}}_{s's_\tau}\mathcal{P}^{{a}}v_S(s')-\mathcal{P}^{{\pi}}_{s's_\tau}\mathcal{P}^{{a}}v_S'(s')\right|+\left|\max\left\{0,(\mathcal{M}^{\mathfrak{g},{\pi}^{\rm DEEP}}Q_S(s_{\tau},{a})-\underset{{a}\in{\mathcal{A}}}{\max}\;\left(R_S(s_\tau,{a}_\tau,g)+\gamma\mathcal{P}^{{a}}_{s's_\tau}v_S'(s')\right)\right\}\right|
\\&\leq \gamma\left\|Pv_S-Pv_S'\right\|
\\&\leq \gamma\|v_S-v_S'\|,
\end{align*}
since for any scalars $a,b,c$ the following holds $
    \left|\max\{a,b\}-\max\{b,c\}\right|\leq \left|a-c\right|$ and the non-expansiveness of $P$.

\textbf{Case 2: }
Let us now consider the case:
\begin{align*}(\mathcal{M}^{\mathfrak{g},{\pi}^{\rm DEEP}}Q_S(s_{\tau},{a})-\underset{{a}\in{\mathcal{A}}}{\max}\;\left(R_S(s_\tau,{a}_\tau,g)+\gamma\mathcal{P}^{{a}}_{s's_\tau}v_S'(s')\right)\geq 0.
\end{align*}

Now first recall that $c>0$, therefore
\begin{align*}
&(\mathcal{M}^{\mathfrak{g},{\pi}^{\rm DEEP}}Q_S(s_{\tau},{a})-\underset{{a}\in{\mathcal{A}}}{\max}\;\left(R_S(s_\tau,{a}_\tau,g)+\gamma\mathcal{P}^{{a}}_{s's_\tau}v_S'(s')\right)
\\&\leq (\mathcal{M}^{\mathfrak{g},{\pi}^{\rm DEEP}}Q_S(s_{\tau},{a})-\underset{{a}\in{\mathcal{A}}}{\max}\;\left(R_S(s_\tau,{a}_\tau,g)+\gamma\mathcal{P}^{{a}}_{s's_\tau}v_S'(s')\right)+c
\\&\leq \left(R_S(s_\tau,{a},g)- c+\gamma\mathcal{P}^{{\pi}}_{s's_\tau}\mathcal{P}^{{a}}v_S(s')\right)|^{{a}\sim{\pi}}-\underset{{a}\in{\mathcal{A}}}{\max}\;\left(R_S(s_\tau,{a}_\tau,g)-c+\gamma\mathcal{P}^{{a}}_{s's_\tau}v_S'(s')\right)
\\&\leq \underset{{a}\in{\mathcal{A}}}{\max}\;\left(R_S(s_\tau,{a},g)-c+\gamma\mathcal{P}^{{\pi}}_{s's_\tau}\mathcal{P}^{{a}}v_S(s')\right)-\underset{{a}\in{\mathcal{A}}}{\max}\;\left(R_S(s_\tau,{a}_\tau,g)- c+\gamma\mathcal{P}^{{a}}_{s's_\tau}v_S'(s')\right)
\\&\leq \gamma\underset{{a}\in{\mathcal{A}}}{\max}\;\left|\mathcal{P}^{{\pi}}_{s's_\tau}\mathcal{P}^{{a}}\left(v_S(s')-v_S'(s')\right)\right|
\\&\leq \gamma\left|v_S(s')-v_S'(s')\right|
\\&\leq \gamma\left\|v_S-v_S'\right\|,
\end{align*}
 using the non-expansiveness of  $P$. Therefore,
\begin{align}
    \left\|(\mathcal{M}^{\mathfrak{g},{\pi}^{\rm DEEP}}Q_S-\underset{{a}\in{\mathcal{A}}}{\max}\;\left[ R_S(\cdot,{a})+\gamma\mathcal{P}^{{a}}v_S'\right]\right\|\leq \gamma\left\|v_S-v_S'\right\|.\label{off_M_bound_gen}
\end{align}
Gathering the results of the three cases proves the statement.

To prove the theorem, we use the following theorem:
\begin{theorem}[Theorem 1, pg 4 in \citep{jaakkola1994convergence}]
Let $\Xi_t(s)$ be a random process that takes values in $\mathbb{R}^n$ and given by the following:
\begin{align}
    \Xi_{t+1}(s)=\left(1-\alpha_t(s)\right)\Xi_{t}(s)\alpha_t(s)L_t(s),
\end{align}
then $\Xi_t(s)$ converges to $0$ with probability $1$ under the following conditions:
\begin{itemize}
\item[i)] $0\leq \alpha_t\leq 1, \sum_t\alpha_t=\infty$ and $\sum_t\alpha_t<\infty$
\item[ii)] $\|\mathbb{E}[L_t|\mathcal{F}_t]\|\leq \gamma \|\Xi_t\|$, with $\gamma <1$;
\item[iii)] ${\rm Var}\left[L_t|\mathcal{F}_t\right]\leq c(1+\|\Xi_t\|^2)$ for some $c>0$.
\end{itemize}
\end{theorem}
\begin{proof}
We need only to prove (i) - (iii) hold. Condition (i) holds by selection of learning rate, hence it remains only to prove (ii) - (iii). We first prove (ii). Let us first 
consider our variant of the Q-learning update rule:
\begin{align*}
Q_{t+1}(s_t,{a}_t)=Q_{t}&(s_t,{a}_t)
\\&+\alpha_t(s_t,{a}_t)\left[\max\left\{(\mathcal{M}^{\mathfrak{g},{\pi}^{\rm DEEP}}Q(s_{\tau_k},{a}), R(s_{\tau_k},{a},g)+\gamma\underset{{a'}\in{\mathcal{A}}}{\max}\;Q_Ss_{t+1},{a'})\right\}-Q_{t}(s_t,{a}_t)\right].
\end{align*}
After subtracting $Q^\star(s_t,{a}_t)$ from both sides and some manipulation we obtain that:
\begin{align*}
&\Xi_{t+1}(s_t,{a}_t)
\\&=(1-\alpha_t(s_t,{a}_t))\Xi_{t}(s_t,{a}_t)
\\&\qquad\qquad\qquad\qquad\;\;+\alpha_t(s_t,{a}_t))\left[\max\left\{(\mathcal{M}^{\mathfrak{g},{\pi}^{\rm DEEP}}Q_Ss_{\tau_k},{a}), R_S(s_{\tau_k},{a},g)+\gamma\underset{a'\in\mathcal{A}}{\max}\;Q_S(s',{a'})\right\}-Q^\star(s_t,{a}_t)\right],  \end{align*}
where $\Xi_{t}(s_t,{a}_t):=Q_t(s_t,{a}_t)-Q^\star(s_t,{a}_t)$.

Let us now define by 
\begin{align*}
L_t(s_{\tau_k},{a}):=\max\left\{(\mathcal{M}^{\mathfrak{g},{\pi}^{\rm DEEP}}Q_Ss_{\tau_k},{a}), R_S(s_{\tau_k},{a},g)+\gamma\underset{a'\in\mathcal{A}}{\max}\;Q_S(s',{a'})\right\}-Q^\star(s_t,a).
\end{align*}
Then
\begin{align}
\Xi_{t+1}(s_t,{a}_t)=(1-\alpha_t(s_t,{a}_t))\Xi_{t}(s_t,{a}_t)+\alpha_t(s_t,{a}_t))\left[L_t(s_{\tau_k},a)\right].   
\end{align}

We now observe that
\begin{align}\nonumber
\mathbb{E}\left[L_t(s_{\tau_k},{a})|\mathcal{F}_t\right]&=\sum_{s'\in\mathcal{S}}P(s';a,s_{\tau_k})\max\left\{(\mathcal{M}^{\mathfrak{g},{\pi}^{\rm DEEP}}Q_Ss_{\tau_k},{a}), R_S(s_{\tau_k},{a},g)+\gamma\underset{a'\in\mathcal{A}}{\max}\;Q_S(s',{a'})\right\}-Q^\star(s_{\tau_k},a)
\\&= T_G Q_t(s,{a})-Q^\star(s,{a}). \label{expectation_L}
\end{align}
Now, using the fixed point property that implies $Q^\star=T_G Q^\star$, we find that
\begin{align}\nonumber
    \mathbb{E}\left[L_t(s_{\tau_k},{a})|\mathcal{F}_t\right]&=T_G Q_t(s,{a})-T_G Q^\star(s,{a})
    \\&\leq\left\|T_G Q_t-T_G Q^\star\right\|\nonumber
    \\&\leq \gamma\left\| Q_t- Q^\star\right\|_\infty=\gamma\left\|\Xi_t\right\|_\infty.
\end{align}
using the contraction property of $T$ established in Lemma \ref{lemma:bellman_contraction}. This proves (ii).

We now prove iii), that is
\begin{align}
    {\rm Var}\left[L_t|\mathcal{F}_t\right]\leq c(1+\|\Xi_t\|^2).
\end{align}
Now by \eqref{expectation_L} we have that
\begin{align*}
  {\rm Var}\left[L_t|\mathcal{F}_t\right]&= {\rm Var}\left[\max\left\{(\mathcal{M}^{\mathfrak{g},{\pi}^{\rm DEEP}}Q_Ss_{\tau_k},{a}), R_S(s_{\tau_k},{a},g)+\gamma\underset{a'\in\mathcal{A}}{\max}\;Q_S(s',{a'})\right\}-Q^\star(s_t,a)\right]
  \\&= \mathbb{E}\Bigg[\Bigg(\max\left\{(\mathcal{M}^{\mathfrak{g},{\pi}^{\rm DEEP}}Q_Ss_{\tau_k},{a}), R_S(s_{\tau_k},{a},g)+\gamma\underset{a'\in\mathcal{A}}{\max}\;Q_S(s',{a'})\right\}
  \\&\qquad\qquad\qquad\qquad\qquad\quad\quad\quad-Q^\star(s_t,a)-\left(T_G Q_t(s,{a})-Q^\star(s,{a})\right)\Bigg)^2\Bigg]
      \\&= \mathbb{E}\left[\left(\max\left\{(\mathcal{M}^{\mathfrak{g},{\pi}^{\rm DEEP}}Q_Ss_{\tau_k},{a}), R_S(s_{\tau_k},{a},g)+\gamma\underset{a'\in\mathcal{A}}{\max}\;Q_S(s',{a'})\right\}-T_G Q_t(s,{a})\right)^2\right]
    \\&= {\rm Var}\left[\max\left\{(\mathcal{M}^{\mathfrak{g},{\pi}^{\rm DEEP}}Q_Ss_{\tau_k},{a}), R_S(s_{\tau_k},{a},g)+\gamma\underset{a'\in\mathcal{A}}{\max}\;Q_S(s',{a'})\right\}-T_G Q_t(s,{a}))^2\right]
    \\&\leq c(1+\|\Xi_t\|^2),
\end{align*}
for some $c>0$ where the last line follows due to the boundedness of $Q$ (which follows from Assumptions 2 and 4). This concludes the proof of the Theorem.

\end{proof}
\end{proof}



\section*{Proof of Proposition \ref{prop:switching_times}}
\begin{proof}
We begin by re-expressing the \textit{activation times} at which the {\fontfamily{cmss}\selectfont Switcher} agent activates DEEPTHINK. In particular,an activation time $\tau_k$ is defined recursively $\tau_k=\inf\{t>\tau_{k-1}|s_t\in A,\tau_k\in\mathcal{F}_t\}$ where $A=\{s\in \mathcal{S},g(s_t)=1\}$.
The proof is given by deriving a contradiction.  Let us there suppose that $\mathcal{M}v_S(s_{\tau_k})\leq v_S(s_{\tau_k})$ and that the activation time $\tau'_1>\tau_1$ is an optimal activation time. Construct the $\mathfrak{g}'$ and $\mathfrak{g}$ policy switching times by $(\tau'_0,\tau'_1,\ldots,)$ and $(\tau'_0,\tau_1,\ldots)$ respectively.  Define by $l=\inf\{t>0;\mathcal{M}v_S(s_t)= v_S(s_t)\}$ and $m=\sup\{t;t<\tau'_1\}$.
By construction, we have that
{\small
\begin{align*}
& \quad v_S(s|{\pi},\mathfrak{g}')
\\&=\mathbb{E}\left[R_S(s_{0},{a}_{0},g)+\mathbb{E}\left[\ldots+\gamma^{l-1}\mathbb{E}\left[R(s_{\tau_1-1},{a}_{\tau_1-1},g)+\ldots+\gamma^{m-l-1}\mathbb{E}\left[ R_S(s_{\tau'_1-1},{a}_{\tau'_1-1},g)+\gamma\mathcal{M}^{{\pi},\mathfrak{g}'}v_S(s_{\tau_1}|{\pi},\mathfrak{g}')\right]\right]\right]\right]
\\&<\mathbb{E}\left[R_S(s_{0},{a}_{0},g)+\mathbb{E}\left[\ldots+\gamma^{l-1}\mathbb{E}\left[ R_S(s_{\tau_1-1},{a}_{\tau_1-1},g)+\gamma\mathcal{M}^{{\pi},\tilde{\mathfrak{g}}}v_S(s_{\tau_1}|{\pi},\mathfrak{g}')\right]\right]\right]
\end{align*}}
We make use of the following observation 
\begin{align}
&\mathbb{E}\left[ R_S(s_{\tau_1-1},{a}_{\tau_1-1},g)+\gamma\mathcal{M}^{{\pi},\tilde{\mathfrak{g}}}v_S(s_{\tau_1}|{\pi},\mathfrak{g}')\right]
\\&\leq \max\left\{\mathcal{M}^{{\pi},\tilde{\mathfrak{g}}}v_S(s_{\tau_1}|{\pi},\mathfrak{g}'),\underset{{a}_{\tau_1}\in{\mathcal{A}}}{\max}\;\left[ R_S(s_{\tau_1},{a}_{\tau_1},g)+\gamma\sum_{s'\in\mathcal{S}}P(s';a_{\tau_1},s_{\tau_1})v_S(s'|{\pi},\mathfrak{g})\right]\right\}.
\end{align}

Using this we deduce that
{
\begin{align*}
&v_S(s|{\pi},\mathfrak{g}')\leq\mathbb{E}\Bigg[R_S(s_{0},{a}_{0},g)+\mathbb{E}\Bigg[\ldots
\\&+\gamma^{l-1}\mathbb{E}\left[ R_S(s_{\tau_1-1},{a}_{\tau_1-1},g)+\gamma\max\left\{\mathcal{M}^{{\pi},\tilde{\mathfrak{g}}}v_S(s_{\tau_1}|{\pi},\mathfrak{g}'),\underset{{a}_{\tau_1}\in{\mathcal{A}}}{\max}\;\left[ R_S(s_{\tau_{k}},{a}_{\tau_{k}},g)+\gamma\sum_{s'\in\mathcal{S}}P(s';a_{\tau_1},s_{\tau_1})v_S(s'|{\pi},\mathfrak{g}),\right]\right\}\right]\Bigg]\Bigg]
\\&=\mathbb{E}\left[R_S(s_{0},{a}_{0},g)+\mathbb{E}\left[\ldots+\gamma^{l-1}\mathbb{E}\left[ R_S(s_{\tau_1-1},{a}_{\tau_1-1},g)+\gamma\left[T_G v_S(s_{\tau_1}|{\pi},\tilde{\mathfrak{g}})\right]\right]\right]\right]=v_S(s|{\pi},\tilde{\mathfrak{g}}),
\end{align*}}
where the first inequality is true by assumption on $\mathcal{M}$. This is a contradiction since $\pi'$ is an optimal policy for \textbf{{\fontfamily{cmss}\selectfont Switcher}}. Using analogous reasoning, we deduce the same result for $\tau'_k<\tau_k$ after which deduce the result. Moreover, by invoking the same reasoning, we can conclude that it must be the case that $(\tau_0,\tau_1,\ldots,\tau_{k-1},\tau_k,\tau_{k+1},\ldots,)$ are the optimal switching times. This completes the proof.
\end{proof}
\section{Proof of Theorem \ref{thm:optimal_policy_budget}}
\begin{proof}
The proof of the Theorem is straightforward since by Theorem \ref{theorem:existence}, {\fontfamily{cmss}\selectfont Switcher}'s problem can be solved using a dynamic programming principle. The proof immediately by applying Theorem 2 in \cite{sootla2022saute}.

\end{proof}

\end{document}